\newif\ifabstract
\abstracttrue
 \abstractfalse 
\newif\iffull
\ifabstract \fullfalse \else \fulltrue \fi

\documentclass[11pt]{article}
\usepackage{amsfonts}
\usepackage{amssymb}
\usepackage{amstext}
\usepackage{amsmath}
\usepackage{xspace}
\usepackage{theorem}
\usepackage{graphicx}
\usepackage{url}
\usepackage{graphics}
\usepackage{colordvi}
\usepackage{colordvi}

\advance \oddsidemargin by -0.8in
\newcommand{\myparskip}{3pt}
\parskip \myparskip

        {\hspace*{\fill}$\Box$\par\vspace{4mm}}

\newcommand{\be}{\begin{enumerate}}
\newcommand{\ee}{\end{enumerate}}
\newcommand{\bd}{\begin{description}}
\newcommand{\ed}{\end{description}}
\newcommand{\bi}{\begin{itemize}}
\newcommand{\ei}{\end{itemize}}

\newtheorem{theorem}{Theorem}[section]

\newtheorem{definition}{Definition}[section]

\newenvironment{proof}{\par \smallskip{\bf Proof:}}{\hfill\stopproof}
\def\stopproof{\square}
\def\square{\vbox{\hrule height.2pt\hbox{\vrule width.2pt height5pt \kern5pt
\vrule width.2pt} \hrule height.2pt}}

\renewcommand{\phi}{\varphi}

\mathchardef\hyphen="2D

\usepackage{booktabs} 
\usepackage{amsmath}
\usepackage[linesnumbered, ruled]{algorithm2e}
\usepackage{subcaption}
\usepackage[font=small,labelfont=bf,tableposition=top]{caption}
\edef\oldtt{\ttdefault}
\usepackage[scaled]{beramono}
\usepackage{tabularx}
\usepackage[T1]{fontenc}
\renewcommand*\ttdefault{\oldtt}
\newcommand{\bera}[1]{{\small\fontfamily{fvm}\selectfont #1}}
\usepackage{enumitem}
\usepackage{cleveref}
\theoremstyle{definition}

\SetCommentSty{mycommfont}
\makeatletter
\newcommand\footnoteref[1]{\protected@xdef\@thefnmark{\ref{#1}}\@footnotemark}
\makeatother

\begin{document}

\title{Robust Continuous Co-Clustering\footnote{Submission under reviewing process on an international conference}}
\author{Xiao He\thanks{NEC Laboratories Europe. Email: {\tt xiao.he@neclab.eu}.}\and Luis Moreira-Matias\thanks{NEC Laboratories Europe. Email: {\tt luis.matias@neclab.eu}.}}

\maketitle

\thispagestyle{empty}

\begin{abstract}
Clustering consists on grouping together samples giving their similar properties. The problem of modeling simultaneously groups of samples and features is known as Co-Clustering. This paper introduces \texttt{ROCCO} - a Robust Continuous Co-Clustering algorithm. \texttt{ROCCO} is a scalable, hyperparameter-free, easy and ready to use algorithm to address Co-Clustering problems in practice over massive cross-domain datasets. It operates by learning a graph-based two-sided representation of the input matrix. The underlying proposed optimization problem is non-convex, which assures a flexible pool of solutions. Moreover, we prove that it can be solved with a near linear time complexity on the input size. An exhaustive large-scale experimental testbed conducted with both synthetic and real-world datasets demonstrates \texttt{ROCCO}'s properties in practice: (i) State-of-the-art performance in cross-domain real-world problems including Biomedicine and Text Mining; (ii) very low sensitivity to hyperparameter settings; (iii) robustness to noise and (iv) a linear empirical scalability in practice. These results highlight \texttt{ROCCO} as a powerful general-purpose co-clustering algorithm for cross-domain practitioners, regardless of their technical background.
\end{abstract}

\section{Introduction}
Clustering consists on grouping together objects (\textit{samples}) giving their similar properties (\textit{features}). However, in practical high-dimensional data mining problems, it is often observed that such groups actually share only a subset of such properties - representing a \textit{context} (e.g. peak hour definition in Transportation \cite{khiari2016}). The problem of modeling simultaneously groups of \textit{samples} and \textit{features} is known as bi-clustering or \textbf{Co-Clustering (CoC)} \cite{chic2017,busygin2008}. It outputs a natural interpretability of their results through an explicit definition of a feature subspace (not necessarily exclusive) for each resulting cluster. This characteristic turns it interesting to apply in practice. Real-world examples range different domains, from Biomedicine \cite{chic2017} to Text mining \cite{dhillon2001co} or Marketing/Retail \cite{hofmann1999}.

Similarly to one-sided clustering, existing solutions to the CoC problem typically suffer from (at least) one of the following issues:
\begin{itemize}[wide]
	\item \textbf{Scalability} concerns the algorithm's ability to address large-scale datasets while keeping a reasonable ratio runtime vs. input size. \texttt{SAMBA} \cite{tanay2002} is a notorious example of this issue by scaling its runtime exponentially on the feature space.
	\item \textbf{Model or Hyperparameter Tuning}: Since CoC is an unsupervised learning process, the choice of adequate hyperparameter values often requires human expertise and/or time-consuming processes. Convex Bi-ClusteRing Algorithm (\texttt{COBRA} \cite{chic2017}) is one of the most recently proposed CoC methods and, like many of its predecessors, it is highly sensitive to its hyperparameter setting.
	\item \textbf{Stability}: Many algorithms are initialization-dependent, requiring the usage of heuristics that often incur on re-training models multiple times (e.g. $K$-Means ). Many CoC algorithms (e.g. Spectral Co-Clustering (\texttt{SCC}) \cite{dhillon2001co}) have $K$-Means (i.e. \texttt{KM}) as a building block, inheriting this disadvantage.
	\item \textbf{Noise (In)tolerance}: Real-world datasets often contain highly noisy records, turning many clustering algorithms not directly usable in practice - especially when facing high \textit{noise vs. patterns} ratios \cite{maurus2016}. OP-Cluster \cite{liu2003} is an approach for CoC that evolves from a sequential pattern mining approach, requiring clusters to have a considerable \textit{support} to be found.
\end{itemize}

This paper introduces \texttt{ROCCO} - a Robust Continuous Co-Clustering algorithm. \texttt{ROCCO} is a scalable, robust, easy and ready to use algorithm to address CoC problems in practice over massive - and possible noisy - cross-domain datasets. It operates by learning a graph-based two-sided representation of the input matrix with a good CoC structure. The underlying proposed optimization problem is non-convex - which assures a flexible pool of solutions. Moreover, we prove that it can be solved with an empirical linear time complexity on the input size. \newline

Fig. \ref{fig:intro_heatmap} illustrates a running example of \texttt{ROCCO} vs. a State-of-the-Art (SoA) benchmarking approach, \texttt{COBRA}. Like \texttt{ROCCO}, \texttt{COBRA} learns a data representation to address CoC problem on high-dimensional datasets - particularly focused on Biomedicine applications. However, it formulates the underlying optimization problem as a \textit{convex} one. While this assumption may hold in some domains, in others it may require a large effort on data preparation to achieve acceptable results. Moreover, in opposition to \texttt{ROCCO}, \texttt{COBRA} is also sensitive to its hyperparameter setting - regarding the number of clusters to be found. The progressive example in Fig. \ref{fig:intro_heatmap} clearly demonstrates that, under a realistic scenario of addressing an input data matrix with a CoC pattern surrounded by the significant amount of noise, \texttt{ROCCO} 's non-convex approach is more suitable and, therefore, a more generalist solution for this problem. Note that, for this benchmark, the \texttt{COBRA} hyperparameters were tuned using the heuristic originally proposed by the authors in \cite{chic2017}. \newline

The contributions of this paper are three-fold:
\begin{enumerate}
	\item \textbf{Problem Formulation:} A generic problem formulation that includes a clear, continuous and non-convex objective function to perform Co-Clustering;
	\item \textbf{Scalability:} A pragmatic solving procedure for learning proved to have theoretical near-linear time complexity on the size of the input matrix. Moreover, we empirically demonstrate that its runtime scales linearly in practice;
	\item \textbf{Robustness:} All-in-all, a CoC algorithm that finds co-clusters in a nearly hyperparameter-free fashion, exhibiting low sensitivity to different settings when evaluated empirically. 
\end{enumerate}
Over the next sections of this paper, we formally introduce the proposed algorithm, the mathematical proofs that sustain its properties and a comprehensive experimental setup that benchmarks \texttt{ROCCO} against the SoA on this topic resorting to multiple large-scale synthetic datasets, as well as to cross-domain ones.
\begin{figure}[!t]
	\centering
	\begin{subfigure}[b]{0.367\linewidth}
		\centering
		\includegraphics[trim=1.1cm 1.28cm 2.18cm 0.0cm, width=\linewidth,clip=TRUE]{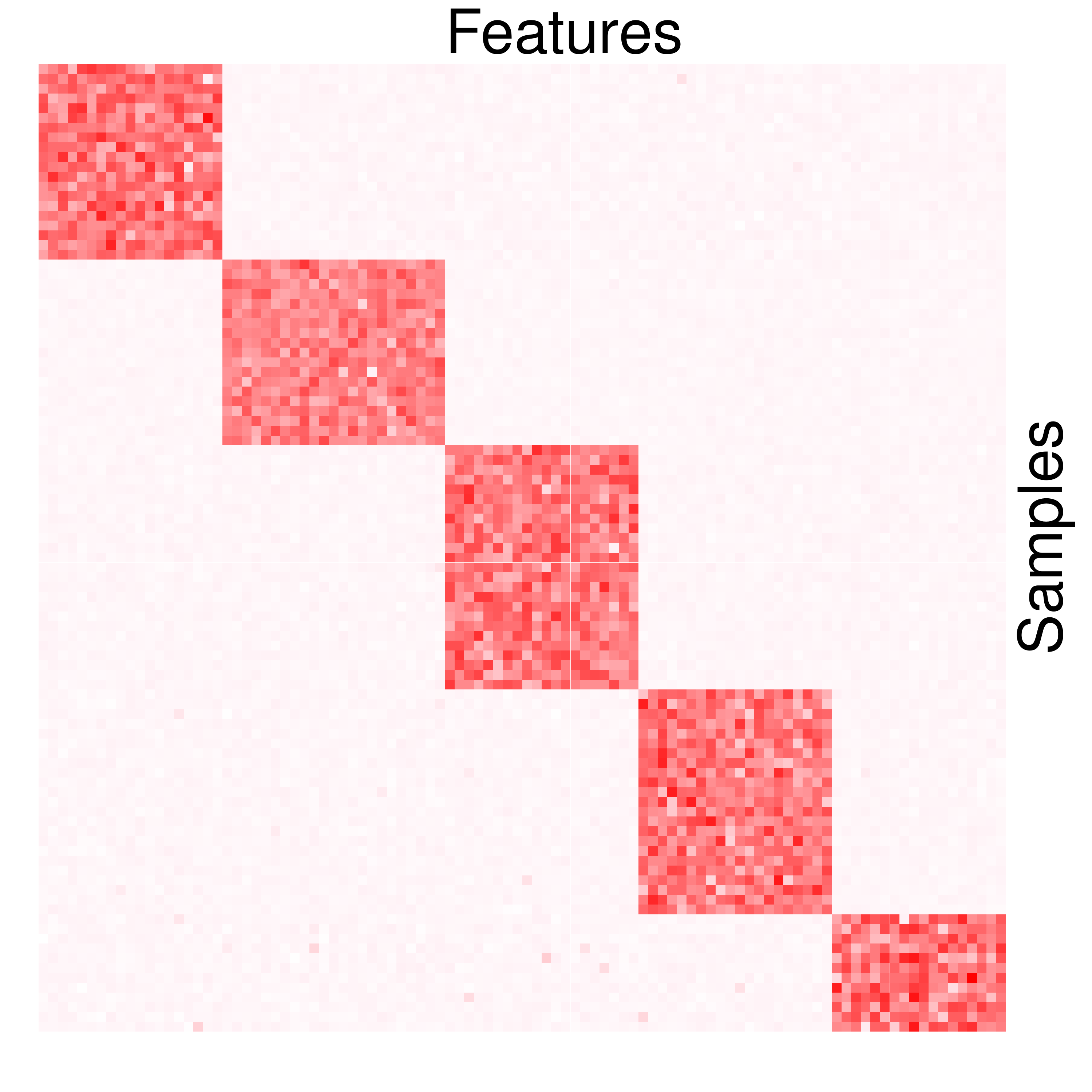}
		\label{fig1-1} 
		\vspace{-0.6cm}\caption{Initial \textit{Toy} Dataset.}
	\end{subfigure}\hspace{0.01cm}
	\begin{subfigure}[b]{0.395\linewidth}
		\centering
		\includegraphics[trim=1.1cm 1.28cm 0.16cm 0.0cm, width=\linewidth,clip=TRUE]{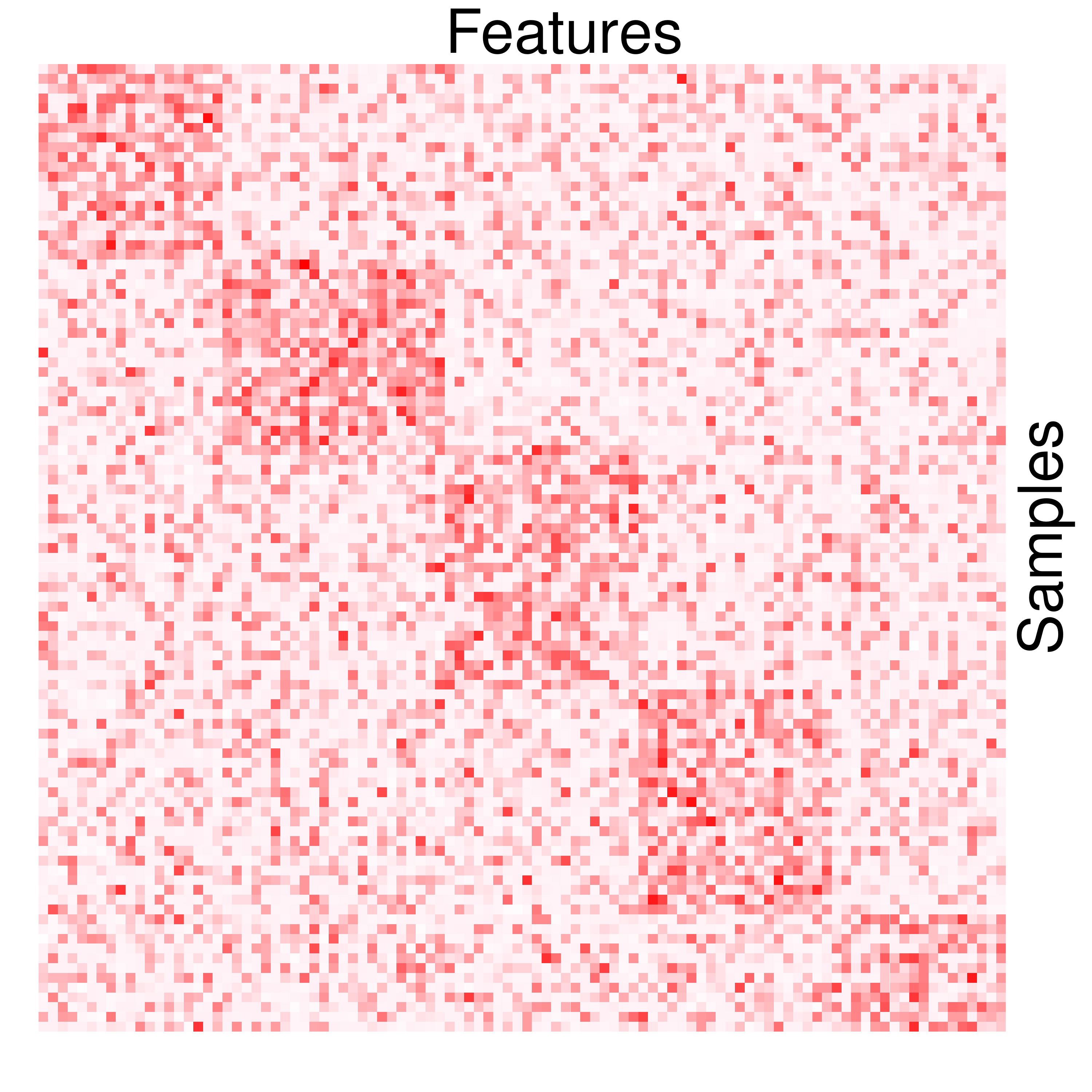}  
		\label{fig1-2} 
		\vspace{-0.6cm}\caption{\textit{Toy} Dataset with added noise.}
	\end{subfigure}\\ \vspace{-0.05cm} 
	\begin{subfigure}[b]{0.367\linewidth}
		\centering
		\includegraphics[trim=1.1cm 0.75cm 2.18cm 1.7cm, width=\linewidth,clip=TRUE]{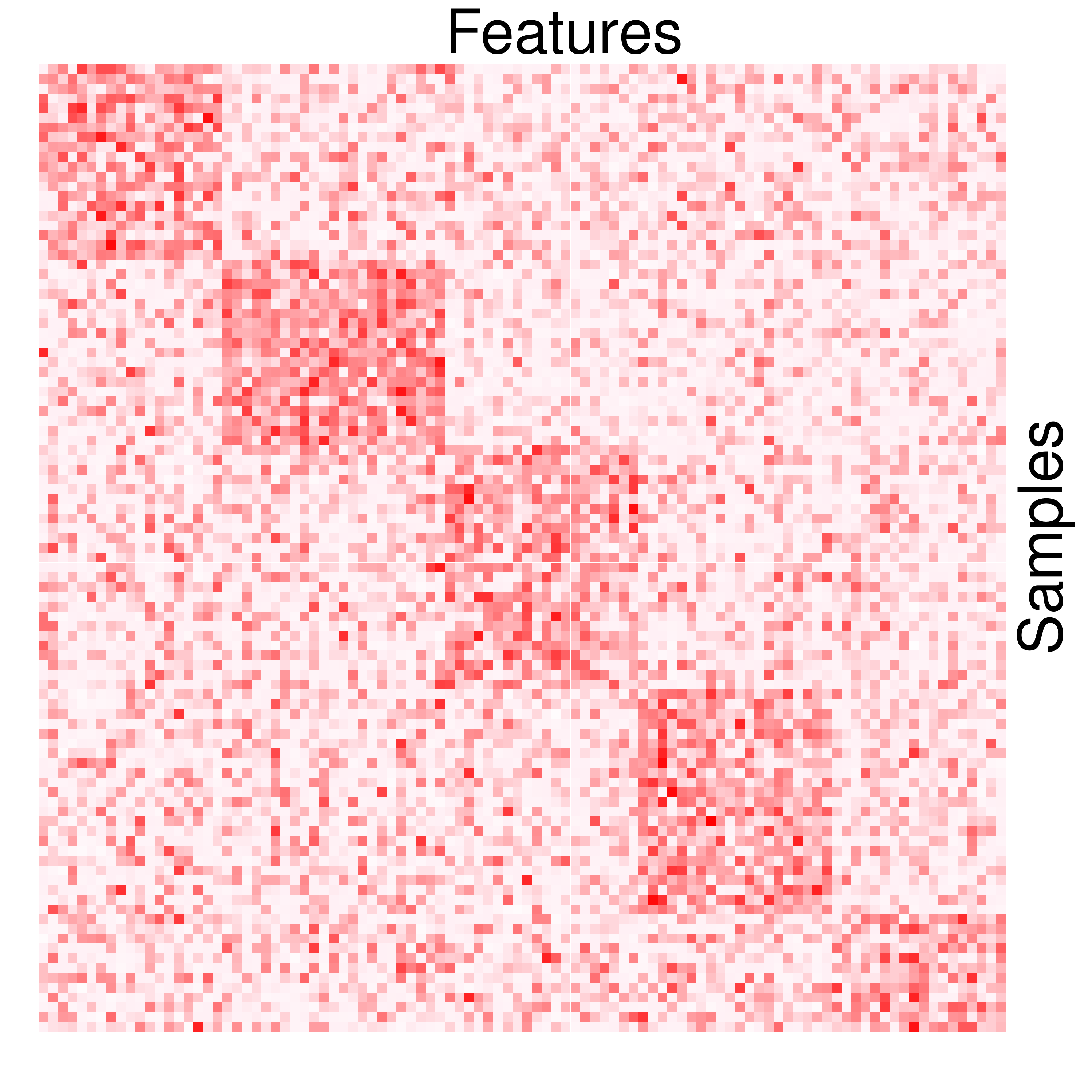}
		\label{fig1-3} 
		\vspace{-0.6cm}\caption{SoA \cite{chic2017} representation.}
	\end{subfigure}\hspace{0.01cm}
	\begin{subfigure}[b]{0.395\linewidth}
		\centering
		\includegraphics[trim=1.1cm 0.75cm 0.16cm 1.7cm, width=\linewidth,clip=TRUE]{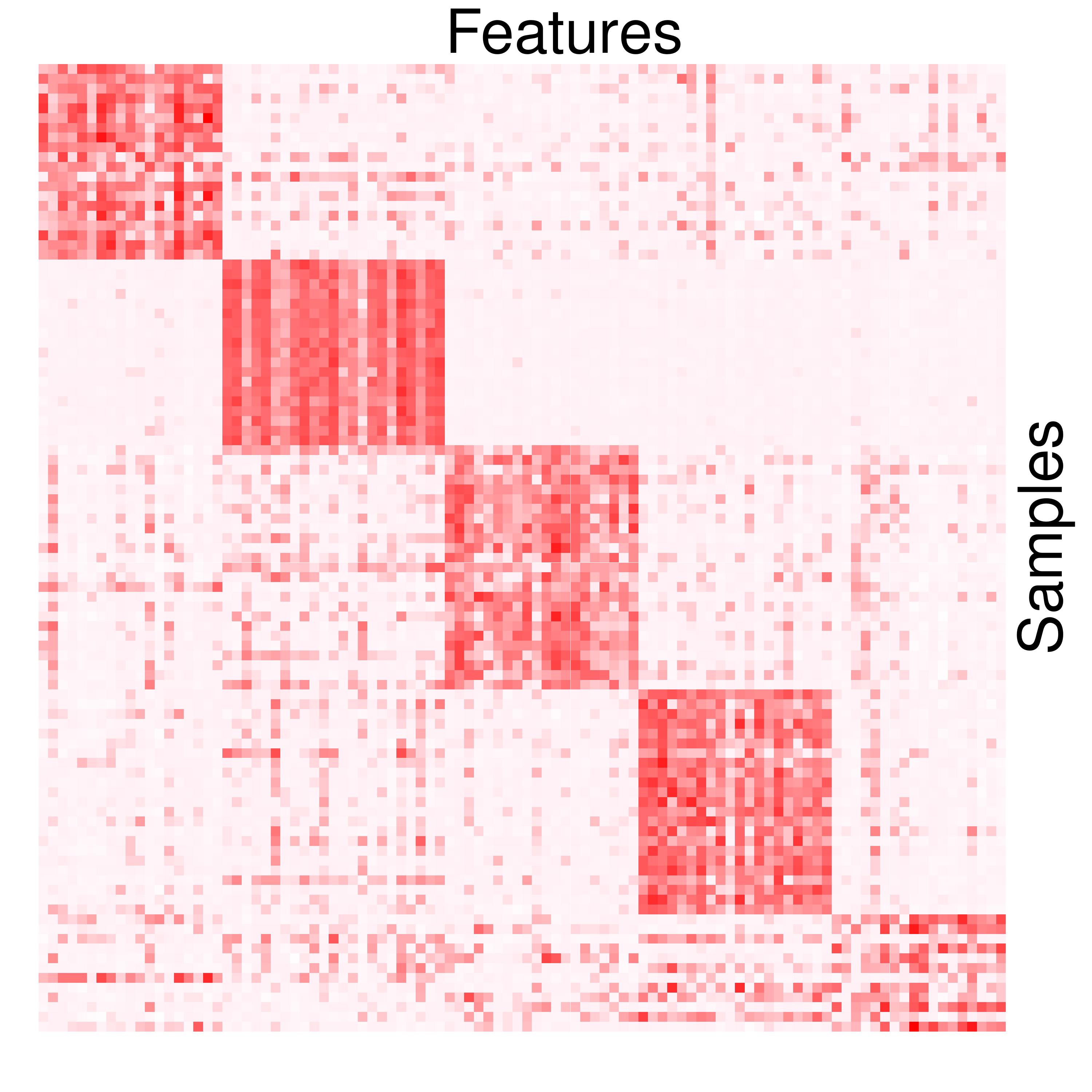}
		\label{fig1-4} 
		\vspace{-0.6cm}\caption{Our representation.}
	\end{subfigure}
	\caption{ROCCO's benchmarking example.}
	\label{fig:intro_heatmap}
\end{figure} 

\section{Robust Continuous Co-Clustering}
\subsection{Problem Formulation}
Consider a dataset with $n$ samples $X=\{x_1,x_2,...,x_n\} \in \mathbb{R}^{n\times p}$ with $p$ features each as input. \texttt{COBRA} \cite{chic2017} is a recent approach to CoC. It settles on a convex optimization problem that learns a graph-based representation of the latent co-clusters structure. It operates in three straightforward steps: (1) firstly, it constructs two graphs $G_p$ and $G_f$ on the samples and features of $X$, e.g. $k$-nearest neighbor on Euclidean space of $X$. Secondly (2), \texttt{COBRA} aims to learn a compact representation $U=\{u_1, u_2,...,u_n\}=\{u^1, u^2,...,u^p\} \in \mathbb{R}^{n\times p}$ of $X$ so that $u_i$ and $u_j$ are similar if samples $i$ and $j$ are connected in $G_p$, $u^i$ and $u^j$ are similar if features $i$ and $j$ are connected in $G_f$. 

This problem can be formulated as follows:

\begin{equation}
\min_U||X-U||_F^2+\frac{\alpha}{2}\sum\nolimits_{i,j\in E_p}||u_i-u_j||_2+\frac{\beta}{2}\sum\nolimits_{i,j\in E_f}||u^i-u^j||_2
\label{eq:cobra}
\end{equation}
where $\alpha$ and $\beta$ balance the relevance of the correspondent terms. Finally, the (3) clustering assignment is based on newly learned representation $U$, where the cluster structure is highlighted.

Despite being elegant, the convexity imposed by Eq. (\ref{eq:cobra}) may not adequately address the patterns in the input data matrix (e.g. see Fig. \ref{fig:intro_heatmap}). Also recently, Shah and Koltun \cite{Shah12092017} proposed a solution for classical clustering problems that overcomes this limitation: the Robust Continuous Clustering (\texttt{RCC}) algorithm. Its problem formulation settles on a non-convex optimization objective function: 	
\begin{equation}
\min_U\frac{1}{2}\sum\nolimits_{i=1}^n ||x_i-u_i||_2^2 + \frac{\alpha}{2}\sum\nolimits_{i,j\in E_p}\rho(||u_i-u_j||_2)
\label{eq:rcc}
\end{equation}
where the non-convexity comes from the regularization function $\rho$. In \cite{Shah12092017}, the authors propose the Geman-McClure function as $\rho$:
\begin{equation}
\rho(y)=\frac{\mu y^2}{\mu+y^2}
\label{eq:rho}
\end{equation}
where $\mu$ controls the function's convexity.

Although the optimization problem is non-convex, \texttt{RCC} achieves SoA clustering accuracy. Moreover, it automatically learns the regularization parameter $\alpha$ and non-convexity parameter $\mu$. The idea of using $L_2$ norm in regularization is to enforce the learned representatives $u_i$ from the same cluster to collapse into a single point. \texttt{RCC} authors \cite{Shah12092017} noted that convex functions have limited robustness to spurious edges in the graph. Consequently, their influence does not diminish during optimization. Therefore, the proposed non-convex function naturally overcomes this issue.

Departing from this idea, we firstly formulate \texttt{ROCCO}'s problem by naturally adapting the optimization objective function of \texttt{RCC} from clustering to the CoC problem. It goes as follows:
\begin{equation}
\min_U||X-U||_F^2+\frac{\alpha}{2}\sum_{i,j\in E_p}\rho(||u_i-u_j||_2)+\frac{\beta}{2}\sum_{i,j\in E_f}\rho(||u^i-u^j||_2)
\label{eq:rocco}
\end{equation}
Moreover, we also choose to inherit the straightforward non-convex function defined in (Eq. (\ref{eq:rho})). 

Departing from the ideas introduced by both \texttt{COBRA} and \texttt{RCC}, \texttt{ROCCO}'s problem formulation introduced in Eq. (\ref{eq:rocco}) is natural and straightforward. However, the solution of such optimization problem in a scalable fashion is far from trivial (e.g. the optimization solver proposed in \texttt{COBRA} cannot address the non-convex problem enunciated in Eq. (\ref{eq:rocco})). In the next section, we propose a novel method to solve Eq. (\ref{eq:rocco}) efficiently.	
\subsection{Optimization}
\label{sec:opt}
The solver originally proposed for \texttt{RCC} is based on the duality between robust estimation and line processes. We follow the same idea to optimize Eq. (\ref{eq:rocco}). To do it so, we start by introducing two auxiliary variables: $L_p(i,j)$ for connections $(i,j) \in E_p$; and $L_f(i,j)$ for connections $(i,j) \in E_f$. Using those, we can re-write Eq. (\ref{eq:rocco}) as:
\begin{equation}
\begin{split}
\min_{U, L_p, L_f}||X-U||_F^2+\frac{\alpha}{2}\sum\nolimits_{i,j\in E_p}(L_p(i,j)||u_i-u_j||_2^2+\Psi(L_p(i,j)))\\
+\frac{\beta}{2}\sum\nolimits_{i,j\in E_f}(L_f(i,j)||u^i-u^j||_2^2+\Psi(L_f(i,j)))
\end{split}
\label{eq:roccodual}
\end{equation}
As shown in \cite{Shah12092017}, Eqs. (\ref{eq:rocco},\ref{eq:roccodual}) are equivalent with respect to the representation $U$ and function $\rho$ when
\begin{equation}
\Psi(z) = \mu(\sqrt{z}-1)^2
\end{equation}

The objective in Eq. (\ref{eq:roccodual}) is biconvex on $(U, L_p)$ and $(U, L_f)$. When $U$ is fixed, the optimal value of $L_p$ and $L_f$ can be derived in closed form. If variables $L_p$ and $L_f$ are fixed, Eq. (\ref{eq:roccodual}) turns into solving a Sylvester equation. Therefore, we optimize the objective by alternating optimization. When $U$ is fixed, we can take the partial derivative of Eq. (\ref{eq:roccodual}) regarding $L_p$ and $L_f$ and set them to zero. This way, we get the update rules for $L_p$ and $L_q$ as follows:
\begin{equation}
\begin{split}
L_p(i,j)=(\frac{\mu}{\mu+||u_i-u_j||^2_2})^2\\
L_f(i,j)=(\frac{\mu}{\mu+||u^i-u^j||^2_2})^2
\label{eq:update}
\end{split}
\end{equation}

If we only consider $L_p$ in Eq. (\ref{eq:roccodual}) and discard the $L_f$ part, we can rewrite Eq. (\ref{eq:roccodual}) with $L_p$ fixed as:
\begin{equation}
\min_U\frac{1}{2}||X-U||_F^2+\frac{\alpha}{2}\sum\nolimits_{i,j\in E_p}L_p(i,j)||U(e_i-e_j)||_2^2
\label{eq:roccop}
\end{equation}
where $e_i \in \mathbb{R}^n$ is an indicator vector with the $i$th element set to $1$ and others set to $0$. Notoriously, the problem enunciated in Eq. (\ref{eq:roccop}) is similar to the \texttt{RCC} one described in Eq. (\ref{eq:rcc}). 

Taking the derivative of Eq. (\ref{eq:roccop}) and setting it to zero, the solution equals to the following linear least-squares problem:
\begin{equation}
\begin{split}
& AU=X, \text{where}\\
& A = I + \alpha\sum\nolimits_{i,j\in E_p}L_p(i,j)(e_i-e_j)(e_i-e_j)^T
\end{split}
\label{eq:sylvesterA}
\end{equation}
Following a similar rationale, if we only consider $L_f$ in Eq. (\ref{eq:roccodual}), the optimal $U$ can be derived by solving yet another linear least-squares problem:
\begin{equation}
\begin{split}
& UB=X, \text{where}\\
& B = I + \alpha\sum\nolimits_{i,j\in E_f}L_f(i,j)(e^i-e^j)(e^i-e^j)^T
\end{split}
\label{eq:sylvesterB}
\end{equation}
where $e^i \in \mathbb{R}^p$ is an indicator vector.

Therefore, when $L_p$ and $L_f$ are fixed, the optimal $U$ equals to the solution of the following Sylvester equation:
\begin{equation}
AU+UB=2X
\label{eq:sylvester}
\end{equation}
where $A$ and $B$ are defined as in Eq. (\ref{eq:sylvesterA}) and (\ref{eq:sylvesterB}). \newline

In practice, we propose to address Eq. (\ref{eq:rocco}) by performing alternating updates of $L_p$ and $L_f$ (using Eq. (\ref{eq:update})) and solving the Sylvester equation in Eq. (\ref{eq:update}) until convergence. However, classical solutions for the latter problem get expensive when scaled up to massive datasets. A illustrative example of a common solver for the Sylvester Equation is the Bartels-Stewart algorithm \cite{bartels1972solution}, which has a time complexity of $\mathcal{O}(n^3+p^3)$. This issue turns this procedure not easily applicable in large-scale CoC problems.

To address this challenge, we propose to solve the Sylvester equation using the Kronecker product. Departing from the observation that matrices $A$ and $B$ are sparse, Eq. (\ref{eq:sylvester}) can be rewritten as 
\begin{equation}
\begin{split}
C\cdot vec(U) = vec(X), \text{where}\\
C = A\otimes I_p + B\otimes I_n
\end{split}
\label{eq:kronecker}
\end{equation}
Eq. (\ref{eq:kronecker}) is a linear least-squares problem on a sparse matrix $C$. A classical solution for sparse linear systems is the Conjugate Gradient method \cite{hestenes1952methods}. Note that $A$ and $B$ relate to the $k$-nearest neighbor graphs on the data with $n$ samples and $p$ features. Therefore, there are $\mathcal{O}(n)$, $\mathcal{O}(p)$, and $\mathcal{O}(np)$ non-zero entries in $A$, $B$ and $C$, respectively. The classic Conjugate Gradient method runs with the complexity of $\mathcal{O}(n^2p^2)$, which is even worse compared to the Bartels-Stewart algorithm (assuming large values for $n$ and $p$).

In this paper, we recognize that there is a critical property of matrix $C$ in Eq. (\ref{eq:kronecker}): $C$ is symmetric diagonally dominant, which enables the usage of a near linear method to solve Eq. (\ref{eq:kronecker}). 

\begin{definition}{Symmetric Diagonally Dominant}
	
	\noindent Matrix $A$ is symmetric diagonally dominant, if $A$ is symmetric and
	\begin{equation}
	A(i,i) \geq \sum_{i\neq j}|A(i,j)|
	\end{equation}
\end{definition}
The equality holds when $A$ is a Laplacian matrix. 

\begin{theorem}
	Matrix $A$ and $B$ defined in Eq. (\ref{eq:sylvesterA}) and (\ref{eq:sylvesterB}) are Symmetric Diagonally Dominant.
	\label{lm:sdd1}
\end{theorem}

\begin{proof}
	Let $N_{ij} \in \mathbb{R}^{n\times n}$ and $N_{i,j} = (e_i-e_j)(e_i-e_j)^T$, thus $N_{i,j}$ is a Laplacian matrix. Therefore, $L_p(i,j)N_{i,j}$ is a Laplacian matrix as well, since $L_p(i,j) \geq 0$. The sum of Laplacian matrices is also a Laplacian matrix, thus $N = \sum_{i,j \in E_p} L_p(i,j)N_{i,j}$ is also a Laplacian matrix with $N(i,i)=\sum_{i\neq j}|N(i,j)|$. Therefore $A=I+\alpha N$ with $A(i,i)=\sum_{i\neq j}|A(i,j)|+1$ is symmetric diagonally dominant. The same proof applies to $B$ as well.
\end{proof}
\begin{theorem}
	Matrix $A\otimes I_p$ is symmetric diagonally dominant if $A \in \mathbb{R}^{n\times n}$ is symmetric diagonally dominant.
	\label{th:sdd2}
\end{theorem}
\begin{proof}
	$A\otimes I$ is symmetric, since both $A$ and $I_p$ are symmetric. Consequently, we have
	\begin{equation}
	M = A \otimes I = I \otimes A = 
	\begin{bmatrix}
	A  & \dots & 0 \\
	\vdots   & \ddots & \vdots \\
	0   & \dots & A
	\end{bmatrix} 
	\end{equation}
	where $A(i,i) \geq \sum_{i\neq j}|A(i,j)|$, thus $M(i,i) \geq \sum_{i\neq j}|M(i,j)|$. This turns $A\otimes I_p$ to be a symmetric diagonally dominant matrix.
\end{proof}   

As $C$ is symmetric diagonally dominant, many theoretic studies have shown that Eq. (\ref{eq:kronecker}) can be solved in near linear time (i.e. $\mathcal{O}(m\log(m))$) with respect to the $m$ non-zero entries in $C$ \cite{koutis2011nearly,cohen2014solving,kelner2013simple}. Besides these theoretic solvers, there are already toolboxes that can solve Eq. (\ref{eq:kronecker}) with empirically linear observation (e.g. Combinatorial MultiGrid (CMG) algorithm \cite{KoutisMT11}). 

In this paper, we propose to use CMG to solve Eq. (\ref{eq:kronecker}). 	
\subsection{Algorithm's Overview}
\label{sec:algorithm}
\texttt{ROCCO} is summarized in Algorithm \ref{alg:rocco}. As input, \texttt{ROCCO} takes a data matrix $X \in \mathbb{R}^{n\times p}$. Its output are clustering assignments of both samples $C_p$ and features $C_f$. Firstly, it constructs the graphs on samples $G_p$ and features $G_f$ separately. Following \cite{Shah12092017}, we also propose to use mutual $k$-nearest neighbor connectivity \cite{brito1997connectivity}, which is more robust than the commonly used $k$-NN graphs. \newline

\subparagraph*{\textbf{Initialization.}} We initialize the representation $U$ to $X$. The auxiliary variables $L_p(i,j)$ and $L_f(i,j)$ are assigned to $1$ whenever the edge $(i,j)$ exists in the graphs $G_p$ and $G_f$, respectively. Finally, the regularization parameters $\alpha$ and $\beta$ are initialized as follows:
\begin{equation}
\alpha = \frac{||X||_2}{||P||_2}, \beta = \frac{||X||_2}{||Q||_2}
\label{eq:reg}
\end{equation}
where $P =\sum\nolimits_{i,j\in E_p}L_p(i,j)(e_i-e_j)(e_i-e_j)^T$ and $Q=\sum\nolimits_{i,j\in E_f}L_f(i,j)$ $(e^i-e^j)(e^i-e^j)^T$. Eq. (\ref{eq:reg}) balances both the data and pairwise terms in Eq. (\ref{eq:rocco}). \newline

\subparagraph*{\textbf{Iteration.}} Then, \texttt{ROCCO} iterates on optimizing the proposed objective function by updating $L_p$, $L_f$ and $U$ alternatively until convergence (as described in Section \ref{sec:opt}). The parameter $\mu$ that controls the (non-)convexity of $\rho$ is proposed to be updated using graduated non-convexity. It begins with a locally convex approximation of the objective, obtained by setting a large enough $\mu$. Then, $\mu$ is automatically decreased every four iterations in order to introduce non-convexity into the objective on a gradual fashion. Under certain assumptions, such continuation schemes are known to attain solutions that are close to the global optimum \cite{mobahi2015theoretical}. Eq. (\ref{eq:reg}) is applied to update $\alpha$ and $\beta$ after every update of $\mu$, so that $\alpha$ and $\beta$ are automatically determined.
\begin{algorithm}[!t]
	\SetNoFillComment
	\SetKwInOut{Input}{Input}
	\SetKwInOut{Output}{Output}
	\Input{Data matrix $X \in \mathbb{R}^{n\times p}$.}
	\Output{Co-Cluster assignments $C_p$, $C_f$.}
	Construct mutual $K$-nearest neighbor graphs $E_p$ and $E_f$\; 
	Initialize $U=X$, $L_p(i,j)=1$, $L_f(i,j)=1$, $\alpha=\frac{||X||_2}{||P||_2}$, $\beta=\frac{||X||_2}{||Q||_2}$, $\mu_p \gg \max||x_i-x_j||^2_2$, $\mu_f \gg \max||x^i-x^j||^2_2$\;
	\While{not converge}{
		Update $L_p$ and $L_f$ using Eq. (\ref{eq:update})\;
		Update $A$ and $B$ using Eq. (\ref{eq:sylvesterA}) and (\ref{eq:sylvesterB})\;
		Update $U$ by solving Eq. (\ref{eq:kronecker}) using CMG \cite{KoutisMT11}\;
		Every four iterations, update $\alpha=\frac{||X||_2}{||P||_2}$, $\beta=\frac{||X||_2}{||Q||_2}$, $\mu_p=\frac{\mu_p}{2}$, $\mu_f=\frac{\mu_f}{2}$
	}
	
	\uIf{X is sparse}{
		$C_p$ = \texttt{KM}($U$)\tcp*{$K$ chosen by Silhouette Coef. \cite{rousseeuw1987silhouettes}}
		$C_f$ = \texttt{KM}($U^T$)\tcp*{$K$ chosen by Silhouette Coef. \cite{rousseeuw1987silhouettes}}
	}
	\Else{
		$C_p$ = \texttt{RCC}($U$) \;
		$C_f$ = \texttt{RCC}($U^T$) \;
	} 
	
	\Return Co-cluster assignments $C_p$ and $C_f$\;
	\caption{\texttt{ROCCO}}
	\label{alg:rocco}
\end{algorithm}
\newline \subparagraph*{\textbf{Co-Cluster Assignment.}}\texttt{RCC} authors \cite{Shah12092017} suggest to derive the clustering assignment automatically from the connected component of a new graph calculated on $U$ with a heuristic-based threshold. However, we empirically observed that this heuristic does not work for \texttt{ROCCO}. The main issue is that the threshold is calculated on the original input space, pushing the output to provide multiple single point clusters. To tackle this challenge, we propose to employ \texttt{RCC} explicitly on both rows and columns of the new representation $U$ to achieve the final clustering assignment. Intuitively, when facing a sparse input matrix $X$, we may again expect issues with the threshold originally proposed in \cite{Shah12092017}\footnote{\label{note1}Notoriously, the empirical results in Section \ref{Results} confirm this expected behavior.}. To mitigate this issue, we suggest to firstly employ a sparsity test on $X$. If positive, we then propose to simply adopt \texttt{KM} instead for final clustering assignment. The number of clusters $K$ is thereby determined by the Silhouette Coefficient \cite{rousseeuw1987silhouettes}. \newline

\subparagraph*{\textbf{Complexity Analysis.}}
The runtime complexity of \texttt{ROCCO} is three-fold: (1) graph construction, (2) representation learning, and (3) clustering assignment. For any input dataset $X \in \mathbb{R}^{n\times p}$, (1) the mutual $k$-nearest neighbor graph on either samples or features runs naively $\mathcal{O}(n^2p)$ and $\mathcal{O}(np^2)$, respectively. However, this step can be easily parallelized and there are many studies for approximated $k$-NN graphs construction which can significantly accelerate the runtime in practice \cite{dong2011efficient, chen2009fast}. The computational complexity of representation learning (2) is reduced to solving the Sylvester equation in Eq. (\ref{eq:rocco}). Theoretically, it can be solved in $\mathcal{O}(np\mbox{ }log(np))$. However, empirically\footnoteref{note1} we observe $\mathcal{O}(np)$. Finally, the co-clustering assignment procedure (3) runs either in $\mathcal{O}(np)$ or $\mathcal{O}(n^2p+np^2)$ using \texttt{RCC} and \texttt{KM}, respectively. Note that \texttt{KM} scales linearly, but the Silhouette-based model selection procedure scales quadratically.

\section{Experiments}
We aim to evaluate the benefits of \texttt{ROCCO} when compared	to other SoA methods for CoC by answering the following questions:
\begin{enumerate}[label=(\textbf{Q\arabic*})]
	\item Is the representation learned by \texttt{ROCCO} suitable for CoC?
	\item Is \texttt{ROCCO}'s clustering assignment fairly good in practice?
	\item Can \texttt{ROCCO} handle noisy and/or sparse inputs? 
	\item Is \texttt{ROCCO} robust to its hyperparameter setting?
	\item Does \texttt{ROCCO} empirically scale linearly in the input size?
\end{enumerate}
To handle these questions, we prepared an extensive empirical experimental setup. Its settings are presented below, followed by the testbed details and their results. Finally, we present a comprehensive discussion on the answers that they provide.

\subsection{Experimental Setup}
We briefly present the comparison partners and their settings, the datasets used to evaluate those, as well as the employed metrics. \newline
\subsubsection{SoA Benchmarkers}
To benchmark \texttt{ROCCO}, we employed 7 different methods from SoA in (automated) clustering (3) and CoC (4). They are enumerated as follows:\vspace{2pt}
\begin{enumerate}[label=(\textbf{\roman*}), wide=0pt]
	\item \textbf{\texttt{KM}}: $K$-means clustering: probably, the most widely used clustering algorithm;\vspace{2pt}
	\item \textbf{\texttt{HDBS}}: Hierarchical Density-Based Clustering \cite{campello2013density} eliminates \texttt{DBSCAN} \cite{ester1996density} need to cut-off the resulting dendrogram. Instead, the final clusters are automatically determined by traversing a tree of condensed view splits, allowing clusters of different densities. \vspace{2pt}
	\item \textbf{\texttt{RCC}}: Robust Continuous Clustering \cite{Shah12092017} learns a regularized graph representation for classical clustering. \vspace{4pt}
	\item \textbf{\texttt{SCC}}: Spectral Co-Clustering \cite{dhillon2001co} treats the input data matrix as a bipartite graph and performs a normalized cut on the graph for clustering assignment.\vspace{2pt}
	\item \textbf{\texttt{SBC}}: Spectral Bi-Clustering \cite{kluger2003spectral} performs Single Value Decomposition (SVD) to approximate the original matrix, followed by \texttt{KM} to find the sample and feature clusters.\vspace{2pt}
	\item \textbf{\texttt{COBRA}}: Convex Bi-Clustering \cite{chic2017} learns a new representation regularized by both sample and feature graphs with a convex function. \vspace{2pt}
	\item \textbf{\texttt{COBRAr}}: To make a fair comparison, we also perform \texttt{RCC} or \texttt{KM} on the learned representation of \texttt{COBRA} for cluster assignment - following strictly the \texttt{ROCCO}'s specific procedure for this task. 		\newline
\end{enumerate}
\subparagraph*{\textbf{Implementations.}} We used Python implementations of \texttt{KM}, \texttt{SCC} and \texttt{SBC} from the package \bera{SciKit-Learn} \cite{scikit-learn}; a fast Python implementation of HDBSCAN from McInnes and Healy \cite{mcinnes2017accelerated}; a Matlab implementation of \texttt{RCC} \cite{Shah12092017} and an R implementation of \texttt{COBRA} \cite{chic2017}. 

\texttt{ROCCO} is implemented in Matlab \footnote{\texttt{ROCCO}'s source code will be made publicly available.}.\newline

\subparagraph*{\textbf{Hyperparameter Setting.}}\texttt{KM}, \texttt{SCC}, and \texttt{SBC} require the number of clusters to be predetermined. \texttt{HDBS} requires to set \bera{min\_cluster\_size} (i.e. the minimum cluster mass), hereby set to $10$. \texttt{RCC}, \texttt{COBRA} and \texttt{ROCCO} learn graphs from the data input. Hereby, all use the mutual $K$-nearest neighbor graph, where $K=10$. Nearest neighbors are calculated by Euclidean distance, except for Text Mining datasets - where cosine distance is used. \texttt{COBRA} is used with the default settings of "\bera{cobra\_validate}" function, where the regularization hyperparameter ($RH$) is automatically set by cross-validation - as suggested by the authors in \cite{chic2017} - over the following grid: $[10^{-3}, 10^{3}]$.

\subsubsection{Datasets}
We employ three different types of datasets: (i) Synthetic and Real-World ones. The latter concern two relevant application domains where CoC problem is common: (ii) Text Mining (documents) and (iii) Biomedicine (gene expression). Details about these types are provided below. 
\begin{enumerate}[label=(\textbf{\roman*}), wide]
	\item \textbf{Synthetic:} We generated synthetic datasets using the  function \bera{make\_biclusters} of Python package \bera{SciKit-Learn} \cite{scikit-learn}. It creates a matrix of small values and implants co-clusters with large values (e.g. Fig. \ref{fig:intro_heatmap}-a). We set both the number of samples and features to $100$, and the number of clusters to $5$. Moreover, we gradually increase the noise level (parameter of \bera{make\_biclusters}) from $10$ to $50$ (e.g. Fig. \ref{fig:intro_heatmap}-b). For each possible configuration pair, we generate $30$ datasets with random seed from $0$ to $29$. 
	
	\item \textbf{Text Mining:} We employ $5$ document datasets. \bera{Classic3} is composed of abstract of papers from three domains and was used in \cite{dhillon2001co}. We follow \cite{dhillon2003information} to generate 2 subsets of \bera{20-Newsgroup} \cite{lang1995newsweeder}: \bera{20news2} and \bera{20news10}. \bera{Reuters} and \bera{RCV1} used in \cite{Shah12092017} are also employed. Each document is represented as the term-frequency inverse document-frequency scores of the top $2000$ most frequent words in their corpus. 
	
	\item \textbf{Biomedicine:} We evaluate on $4$ gene expression datasets preprocessed by Nie \textit{et al.} \cite{nie2010efficient}: \bera{ALLAML}, \bera{Carcinom}, \bera{GLIOMA} and \bera{Lung}. Feature scaling is performed beforehand.
\end{enumerate}
Table \ref{tab:datasets} shows statistics about the real-world datasets (ii,iii).

\subsubsection{Evaluation metrics}
In the synthetic datasets, we know both the true cluster labels for both samples and features. In the real-world, we only know the sample labels. Therefore, we use Normalized Mutual Information (NMI) \cite{vinh2010information} for evaluation, which measures the dependence between true labels and the obtained labels. However, other studies \cite{Shah12092017} point that NMI value is biased on the number of clusters. Consequently, different numbers of clusters are often obtained by the evaluated methods. To address this issue, we further evaluate the clustering performance using Adjusted Mutual Information (AMI) \cite{vinh2010information} and Adjusted Rand Index (ARI) \cite{hubert1985comparing}, whenever the true labels are available.

\subsection{Testbeds and Results}
\label{Results}
In this Section, we describe the testbeds built to answer the proposed research questions, pointing out the obtained results. \newline

\subparagraph*{\textbf{Representation Learning (Q1).}}
In this scenario, we evaluate the goodness of the representation learned by \texttt{ROCCO} for CoC. 

Regarding the synthetic datasets, we compare the learned representations by visualizing the resulting heatmaps in Figs. \ref{fig:intro_heatmap} and \ref{fig:heatmap}. Figs. \ref{fig:intro_heatmap}-(a,b) illustrate well this problem setting: the original synthetic datasets without and with noise of level $50$, respectively. Fig. \ref{fig:intro_heatmap}-(c) already illustrated the behavior of \texttt{COBRA} by tuning its $RH$ as originally suggested in \cite{chic2017} (i.e. $RH=10$). However, we extended this evaluation by plotting \texttt{COBRA}'s representation for different levels of $RH$, together with either the samples/features clustering made by \texttt{RCC}. The resulting heatmaps are depicted in Fig. \ref{fig:heatmap}.
\begin{table}[!t]
	\begin{center}
		\caption{Summary statistics of the real-world datasets.}
		\scalebox{0.89}{\begin{tabular}{l|lrrr}
				\cline{1-5}\noalign{\smallskip}
				\textbf{Domain} & \textbf{Name} & \textbf{Samples} & \textbf{Features} & \textbf{Clusters} \\
				\cline{1-5}\noalign{\smallskip}
				&\bera{Classic3} & 3891 & 2000 & 3 \\
				&\bera{20news2} & 1777 & 2000 & 2 \\
				\textit{\textbf{Text Mining}}&\bera{20news10} & 9607 & 2000 & 10 \\
				&\bera{Reuters} & 9082 & 2000 & 50 \\
				&\bera{RCV1}  & 10000 & 2000 & 4 \\ \hline
				&\bera{ALLAML} & 72 & 7129 & 2\\
				&\bera{Carcinom} & 174 & 9182 & 11\\
				\textit{\textbf{Biomedicine}}&\bera{GLIOMA} & 50 & 4434 & 4\\
				&\bera{Lung}  & 203 & 3312 & 5\\ 
				\cline{1-5}\noalign{\smallskip}
		\end{tabular}}
		\vspace{-0.5cm}
		\label{tab:datasets}
	\end{center}
\end{table}
The real-world datasets are used as follows: \texttt{KM} is performed over the representations $U$ learned by \texttt{ROCCO}, \texttt{RCC} and \texttt{COBRA}. In this case, \texttt{KM} receives the true number of classes as input for either samples and features. \texttt{KM} is also evaluated when using directly over the input matrix, along with \texttt{SCC} and \texttt{SBC}. Similarly, these algorithms are also fed with the real number of classes. If the true number of feature clusters is not known, we assume it to be equal to the samples one. \texttt{HDBS} was deliberately discarded in this evaluation since it is not possible to set the number of clusters apriori for this method. NMI and AMI values for the obtained results are displayed in Table \ref{tab:results}. \newline

\subparagraph*{\textbf{Real-World Problems (Q2,Q3).}} 
Hereby, we evaluate the performance of \texttt{ROCCO} in practice, where the number of clusters is unknown apriori. \texttt{HDBS},  \texttt{RCC},  \texttt{COBRA},  \texttt{COBRAr} and  \texttt{ROCCO} automatically output both row and column clusters. However, \texttt{KM}, \texttt{SCC}, and \texttt{SBC} require the number of clusters to be defined beforehand. To tackle such issue, we adopt a grid search procedure over the number of clusters on $[2, 20]$, using then the Silhouette Coefficient \cite{rousseeuw1987silhouettes} for model selection purposes. The resulting methods are then denoted as \texttt{autoKM}, \texttt{autoSCC}, and \texttt{autoSBC}, respectively. To allow a fair comparison in this particular setup, we employ the sparsity test (i.e. document datasets are typically highly sparse) already proposed in the cluster assignment stage of \texttt{ROCCO} also for \texttt{RCC} and \texttt{COBRAr}.

Again, we devised two evaluation test beds: a (i) synthetic and a (ii) real-world one. The first (i) consists on synthetic datasets with gradually increased noise levels (from $10$ to $50$). The (ii) real-world datasets were used \textit{as is}. Results using the ARI for the clustering procedures over either the samples or the features of the synthetic datasets (i) are depicted in Fig. \ref{fig:Syn}. NMI and AMI values for the results obtained over the (ii) real-world datasets are displayed in Table \ref{tab:results2}. \newline
\begin{figure}[!t]
	\centering
	\begin{subfigure}[b]{0.377\linewidth}
		\centering
		\includegraphics[trim=1.1cm 1.28cm 2.18cm 0.0cm, width=\linewidth,clip=TRUE]{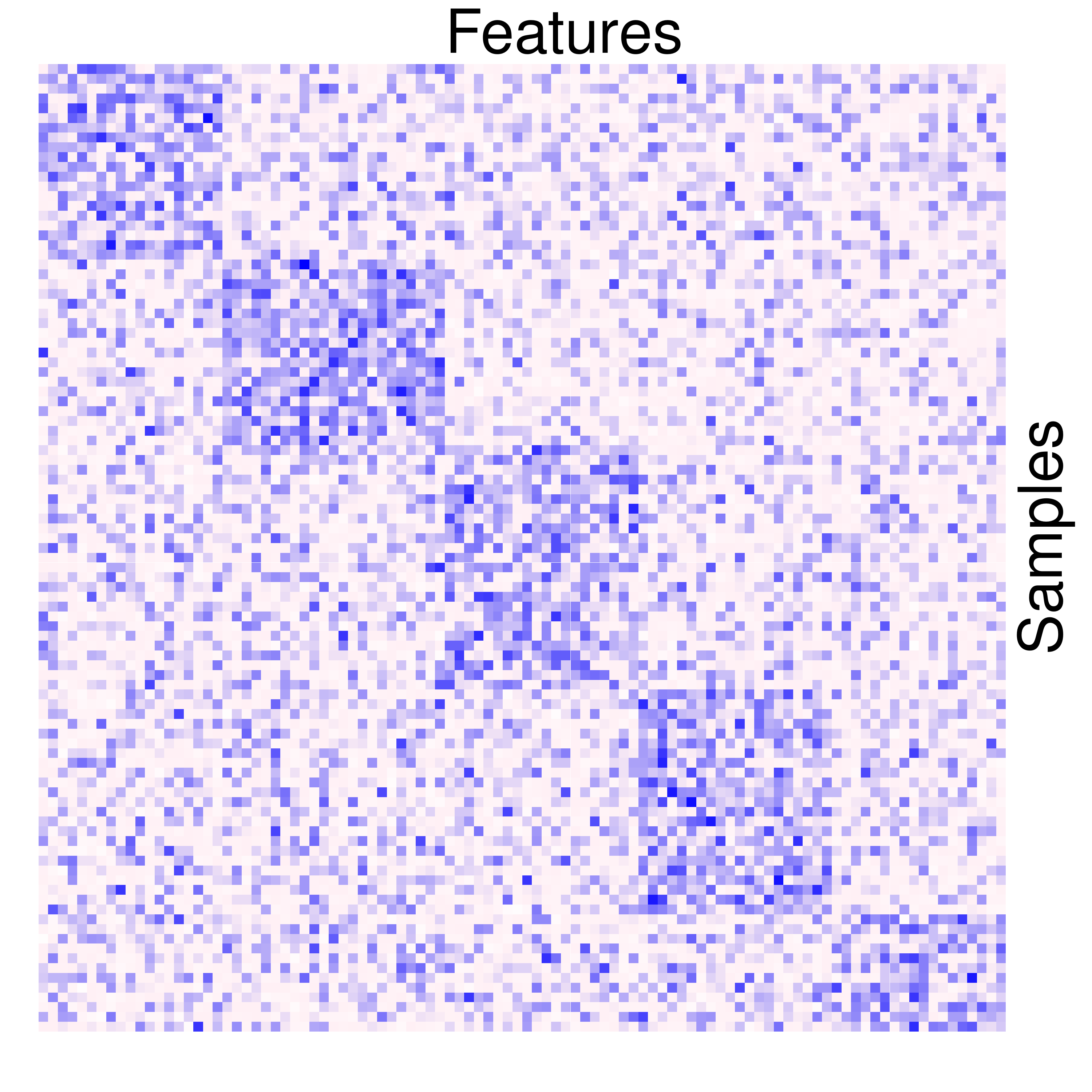}
		\label{figheatmap-1} 
		\vspace{-0.6cm}\caption{\texttt{COBRA} with $RH=1$.}
	\end{subfigure}\hspace{0.01cm}
	\begin{subfigure}[b]{0.395\linewidth}
		\centering
		\includegraphics[trim=1.1cm 1.28cm 0.16cm 0.0cm, width=\linewidth,clip=TRUE]{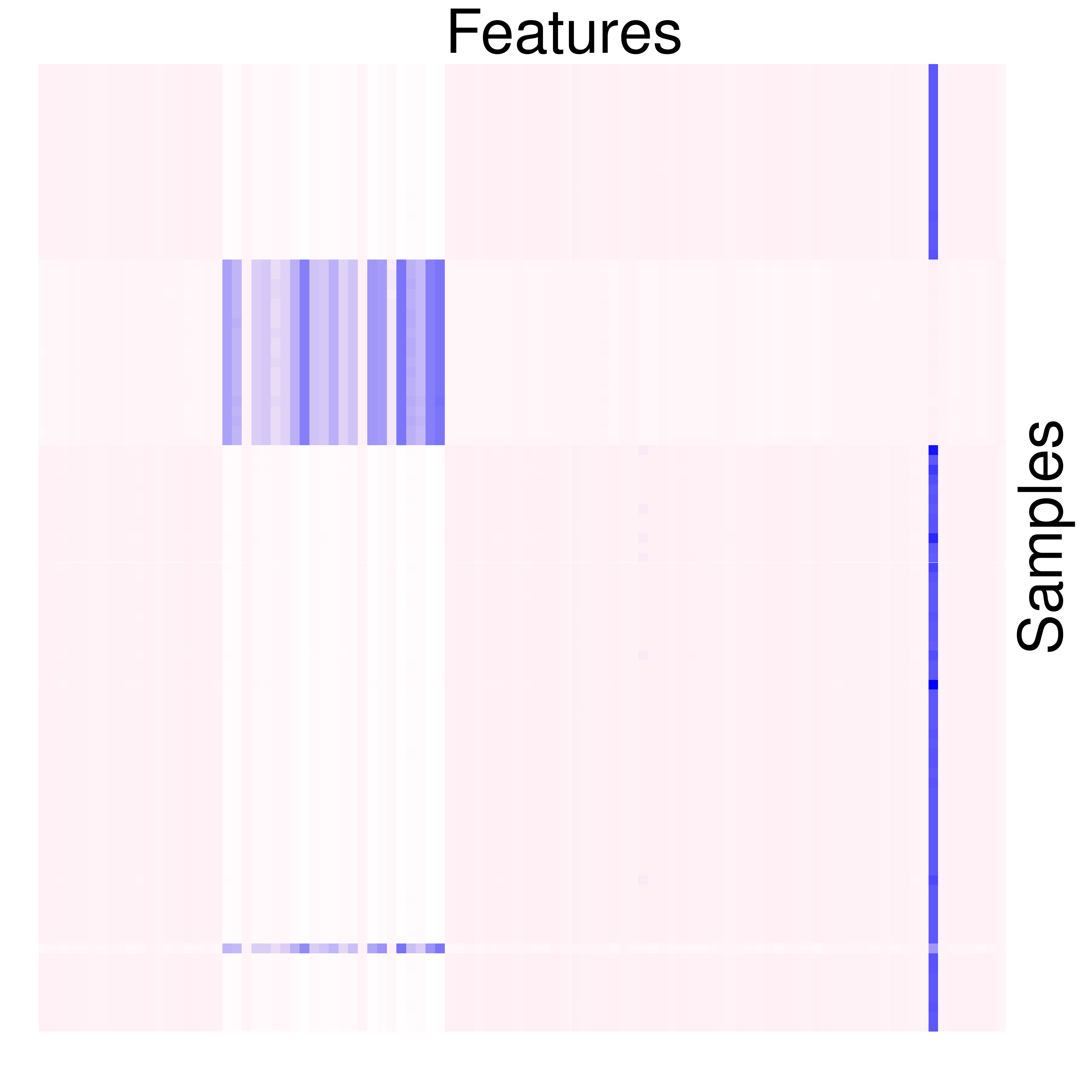} 
		\label{figheatmap-2} 
		\vspace{-0.6cm}\caption{\texttt{COBRA} with $RH=100$.}
	\end{subfigure}\\ \vspace{-0.05cm}
	\begin{subfigure}[b]{0.377\linewidth}
		\centering
		\includegraphics[trim=1.1cm 0.75cm 2.18cm 1.7cm, width=\linewidth,clip=TRUE]{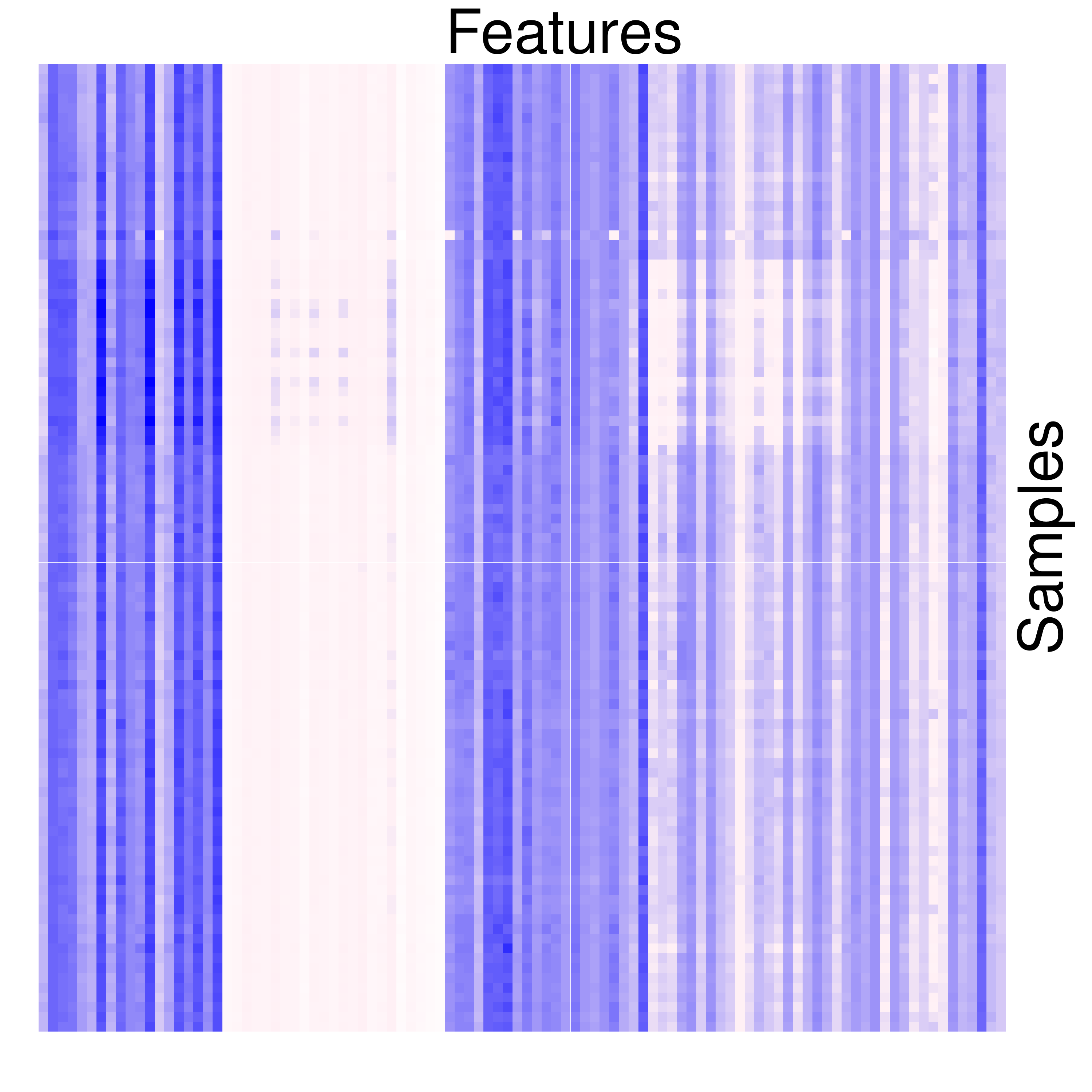}
		\label{figheatmap-3} 
		\vspace{-0.6cm}\caption{\texttt{COBRA} with $RH=1000$.}
	\end{subfigure}\hspace{0.01cm}
	\begin{subfigure}[b]{0.395\linewidth}
		\centering
		\includegraphics[trim=1.1cm 0.75cm 0.16cm 1.7cm, width=\linewidth,clip=TRUE]{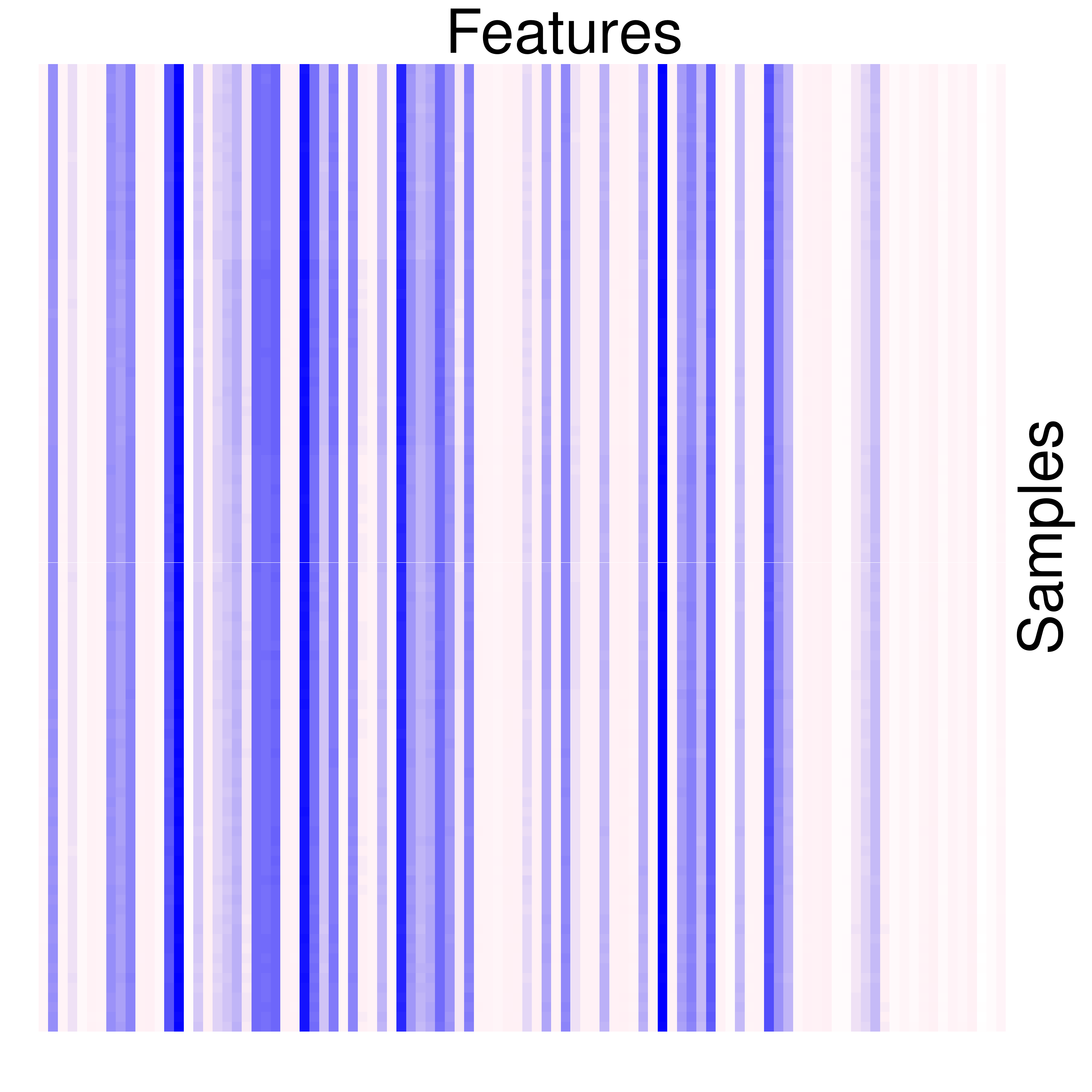}
		\label{figheatmap-4} 
		\vspace{-0.6cm}\caption{\texttt{RCC} over the samples.}
	\end{subfigure}\vspace{-0.05cm}
	\begin{subfigure}[b]{0.377\linewidth}
		\centering
		\includegraphics[trim=1.1cm 0.75cm 2.18cm 1.7cm, width=\linewidth,clip=TRUE,angle=90]{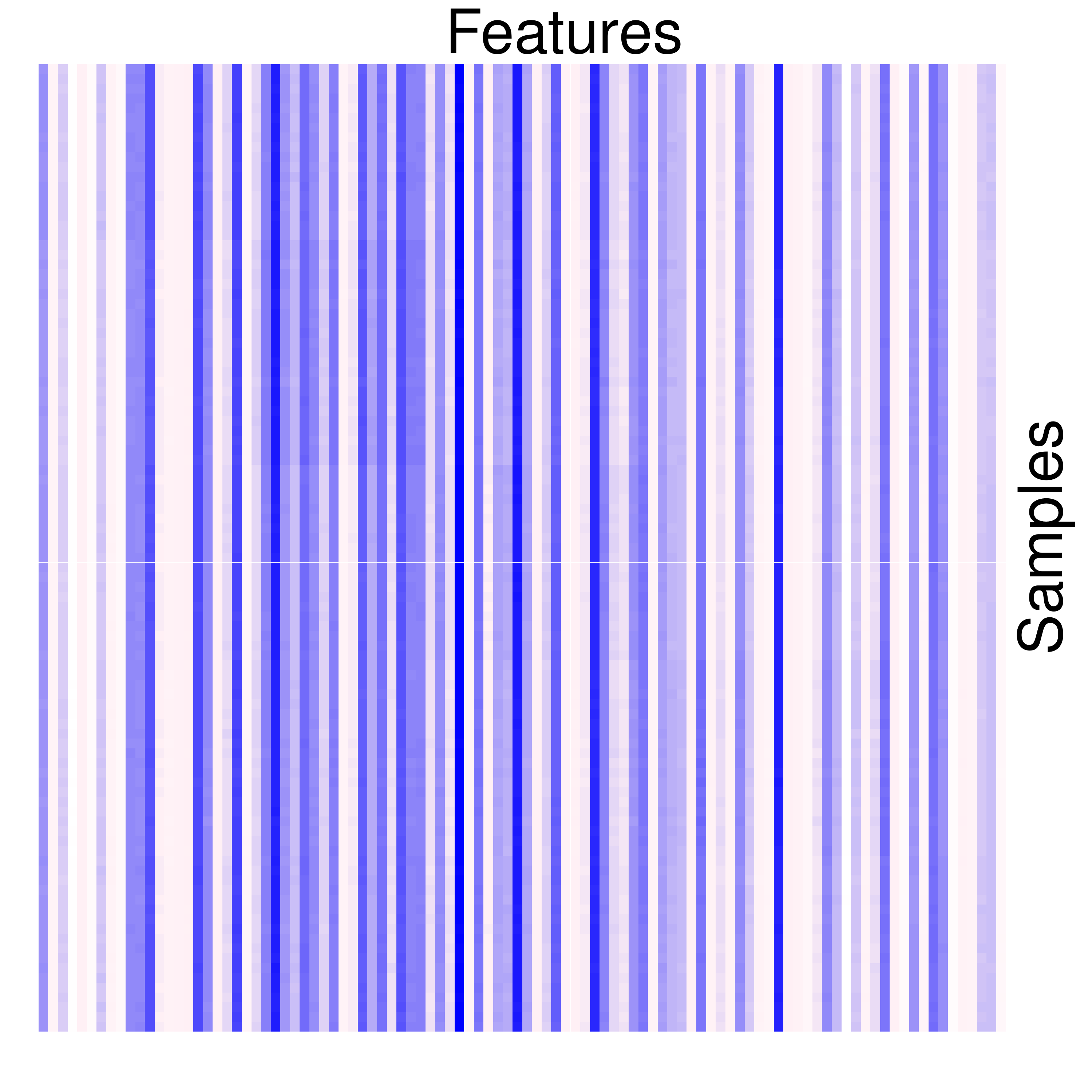}
		\label{figheatmap-5} 
		\vspace{-0.5cm}\caption{\texttt{RCC} over the features.}
	\end{subfigure}\hspace{0.01cm}
	\begin{subfigure}[b]{0.395\linewidth}
		\centering
		\includegraphics[trim=1.1cm 0.75cm 0.16cm 1.7cm, width=\linewidth,clip=TRUE]{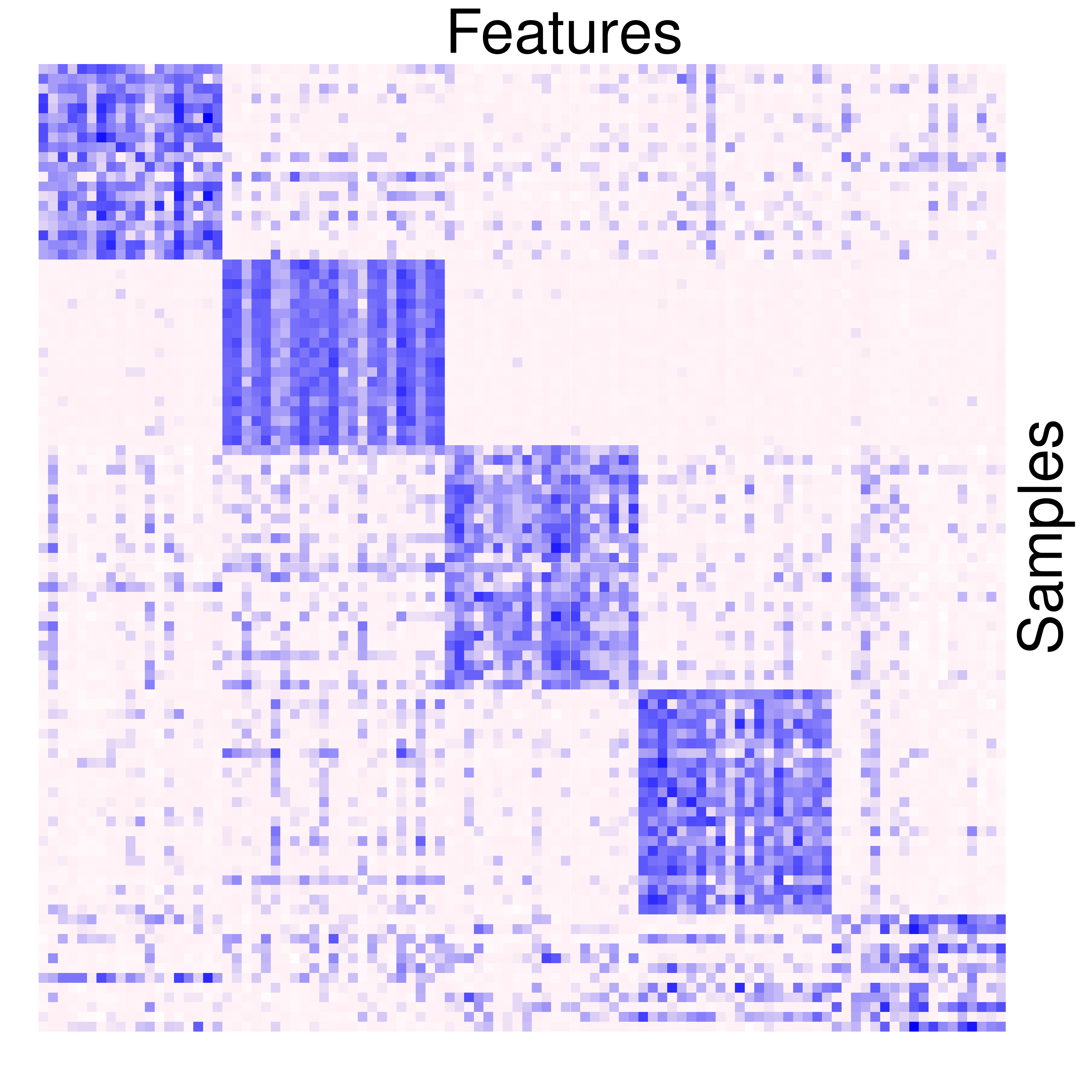}
		\label{figheatmap-6} 
		\vspace{-0.65cm}\caption{\texttt{ROCCO}.}
	\end{subfigure}
	\caption{Comparison between the Representations obtained from \texttt{COBRA}, \texttt{RCC} and \texttt{ROCCO}.} 
	\label{fig:heatmap}
\end{figure}     
\begin{table*}
	\center 
	\caption{Representation Learning Results. \texttt{COBRA}, \texttt{RCC} and \texttt{ROCCO} representations were used as input for \texttt{KM}. True number of clusters was given to all the compared methods. Last row depicts the methods average rank per metric.}
	\vspace{-0.4cm}
	\vspace{0.2cm}
	\centering
	\scalebox{0.7}{\begin{tabular}{|l||cc|cc|cc|cc|cc|cc|cc|cc}
			\hline
			& \multicolumn{2}{c|}{\texttt{KM}} & \multicolumn{2}{c|}{\begin{tabular}{@{}c@{}}\texttt{RCC}\\\texttt{+KM}\end{tabular}} & \multicolumn{2}{c|}{\texttt{SCC}} & \multicolumn{2}{c|}{\texttt{SBC}} & \multicolumn{2}{c|}{\begin{tabular}{@{}c@{}}\texttt{COBRA}\\\texttt{+KM}\end{tabular}} & \multicolumn{2}{c|}{\begin{tabular}{@{}c@{}}\texttt{ROCCO}\\\texttt{+KM}\end{tabular}} \\ \hline
			& \small{NMI} & \small{AMI}  & \small{NMI} & \small{AMI}  & \small{NMI} & \small{AMI}  & \small{NMI} & \small{AMI}  & \small{NMI} & \small{AMI}  & \small{NMI} & \small{AMI}    \\     \hline  \hline
			\bera{Classic3} & 0.706 & 0.703 & 0.909 & 0.908 & 0.908 & 0.905 & 0.911 & 0.908 & 0.716 & 0.714&\textbf{0.915} &  \textbf{0.914} \\ 
			\bera{20news2} & 0.657  & 0.650 & 0.659 & 0.654 & 0.005 & 0.001 & 0.010   & 0.001 & 0.620 &0.618 & \textbf{0.806} & \textbf{0.802} \\ 
			\bera{20news10} & 0.219 & 0.201 & 0.314 & 0.266 & \textbf{0.475} & \textbf{0.396} & 0.264   & 0.192 & 0.123 & 0.102 & 0.438 & 0.394 \\ 
			\bera{Reuters} &  0.421 & 0.307 & 0.475 & 0.362 & 0.408 & 0.316 & 0.137 & 0.077 &  0.446  & 0.332&  \textbf{0.540} & \textbf{0.397} \\ 
			\bera{RCV1} & \textbf{0.404}& \textbf{0.398} & 0.394 & 0.389 & 0.343 & 0.310 & 0.049 & 0.044 &  0.398  & 0.395 &  \textbf{0.403} & 0.393 \\
			\bera{ALLAML} & 0.129 & 0.117 & 0.007 & 0.008 & \textbf{0.189} & \textbf{0.177} & 0.125 & 0.112 &  0.171  & 0.158 &  0.154 & 0.141 \\ 
			\bera{Carcinom} & 0.648 & 0.588 & 0.778 & 0.719 & 0.373 & 0.246 & 0.443  & 0.341 & 0.651 & 0.593 &  \textbf{0.794} & \textbf{0.745} \\ 
			\bera{GLIOMA} &  \textbf{0.537} & \textbf{0.478} & 0.484 & 0.411 & 0.085 & 0.001 & 0.151 & 0.073 &  0.507 &  0.450 &  0.484 & 0.411  \\ 
			\bera{Lung} & 0.658 & 0.557 & 0.649 & 0.557 & 0.391 & 0.368 & 0.451 & 0.366 & 0.661 & 0.562 &   \textbf{0.711} &  \textbf{0.645} \\ \hline
			$\varnothing$ \textbf{Rank}  &3.444  & 3.500 &3.277  &3.166 & 4.444  & 4.166 &4.777  & 5.111 & 3.333 & 3.222 & \textbf{1.722} & \textbf{1.833} \\ \hline 
	\end{tabular}}   
	\label{tab:results}
\end{table*}			
\begin{table*}[!t]
	\center 
	\caption{Real-world Benchmark where the real $K$ is not given. \texttt{autoKM}, \texttt{autoSCC}, and \texttt{autoSBC} follow a grid-search procedure over a pool of $K$ values to then select the best model using the Silhouette Coef.. The remaining methods choose $K$ automatically.}
	\vspace{-0.4cm}
	\vspace{0.2cm}
	\centering
	\scalebox{0.7}{\begin{tabular}{|l||cc|cc|cc|cc|cc|cc|cc|cc|cc|cc}
			\hline
			& \multicolumn{2}{c|}{\texttt{autoKM}} & \multicolumn{2}{c|}{\texttt{HDBS}} & \multicolumn{2}{c|}{\texttt{RCC}} & \multicolumn{2}{c|}{\texttt{autoSCC}} & \multicolumn{2}{c|}{\texttt{autoSBC}} & \multicolumn{2}{c|}{\texttt{COBRA}} & \multicolumn{2}{c|}{\texttt{COBRAr}} & \multicolumn{2}{c|}{\texttt{ROCCO}}
			\\ \hline
			& \small{NMI} & \small{AMI}  & \small{NMI} & \small{AMI}  & \small{NMI} & \small{AMI}  & \small{NMI} & \small{AMI}  & \small{NMI} & \small{AMI}  & \small{NMI} & \small{AMI} & \small{NMI} & \small{AMI} & \small{NMI} & \small{AMI}    \\     \hline  \hline
			\bera{Classic3} & 0.704 & 0.701 & 0.019 & 0.014 & 0.903 & 0.903 & 0.906 & 0.903 & \textbf{0.914} & 0.912 & 0.362 & 0.001 & 0.715 & 0.713 & \textbf{0.914} & \textbf{0.913} \\ 
			\bera{20news2} & 0.293  & 0.169 & 0.052 & 0.050 & 0.391 & 0.333 & 0.005 & 0.001 & 0.010 & 0.001 & 0.303 & 0.001 & 0.358 & 0.212 & \textbf{0.806} & \textbf{0.802} \\ 
			\bera{20news10} & 0.290 & 0.275 & 0.155 & 0.071 & 0.349 & 0.333 & 0.274   & 0.098 & 0.037 & 0.003 & \textbf{0.500} & 0.001 & 0.290 & 0.279 & 0.495 & \textbf{0.479} \\ 
			\bera{Reuters} &  0.271 & 0.160 & 0.172 & 0.071 & 0.252 & 0.146 & 0.280 & \textbf{0.168} &  0.034  & 0.018 &  \textbf{0.479} & 0.001 & 0.270 & 0.159 & 0.295 & 0.163 \\ 
			\bera{RCV1} & 0.466 & 0.311 & 0.165 & 0.140 & 0.454 & 0.325 & 0.346 & 0.313 & 0.012 & 0.001 & 0.367 & 0.001 & \textbf{0.483} & 0.346 & 0.462 & \textbf{0.363} \\
			\bera{ALLAML} & 0.154 & 0.141 & 0.001 & 0.001 & 0.007 & 0.001 & 0.189 & \textbf{0.177} &  0.125  & 0.112 &  0.208 & 0.023 & 0.001 & 0.001 & \textbf{0.279} & 0.158 \\ 
			\bera{Carcinom} & 0.659 & 0.599 & 0.499 & 0.332 & 0.474 & 0.215 & 0.388  & 0.155 & 0.270 & 0.099 &  0.680 & 0.062 & 0.430 & 0.173 & \textbf{0.791} & \textbf{0.743} \\ 
			\bera{GLIOMA} &  \textbf{0.630} & \textbf{0.432} & 0.398 & 0.300 & 0.476 & 0.379 & 0.050 & 0.011 & 0.154 &  0.089 &  0.587 & 0.001 & 0.174 & 0.064 & 0.484 & 0.411  \\ 
			\bera{Lung} & 0.550 & 0.314 & 0.401 & 0.290 & 0.557 & 0.303 & 0.306 & 0.128 & 0.505 & 0.287 & 0.496 &  0.319 & \textbf{0.737} & \textbf{0.605} & 0.602 & 0.389  \\ \hline
			$\varnothing$ \textbf{Rank}  & 3.722 & 3.555 & 6.500 & 5.555 & 4.111 & 3.833 & 5.777 & 4.722  & 6.388 & 5.944 & 3.333 & 6.944 & 4.333 & 4.000 & \textbf{1.833} & \textbf{1.444} \\ \hline 
	\end{tabular} } 
	\label{tab:results2}
\end{table*}        
\begin{figure*}[!h]
	\centering    
	\begin{subfigure}[b]{0.38\linewidth}
		\centering
		\includegraphics[trim=0.0cm 0cm 1cm 2.5cm, width=\linewidth,clip=TRUE]{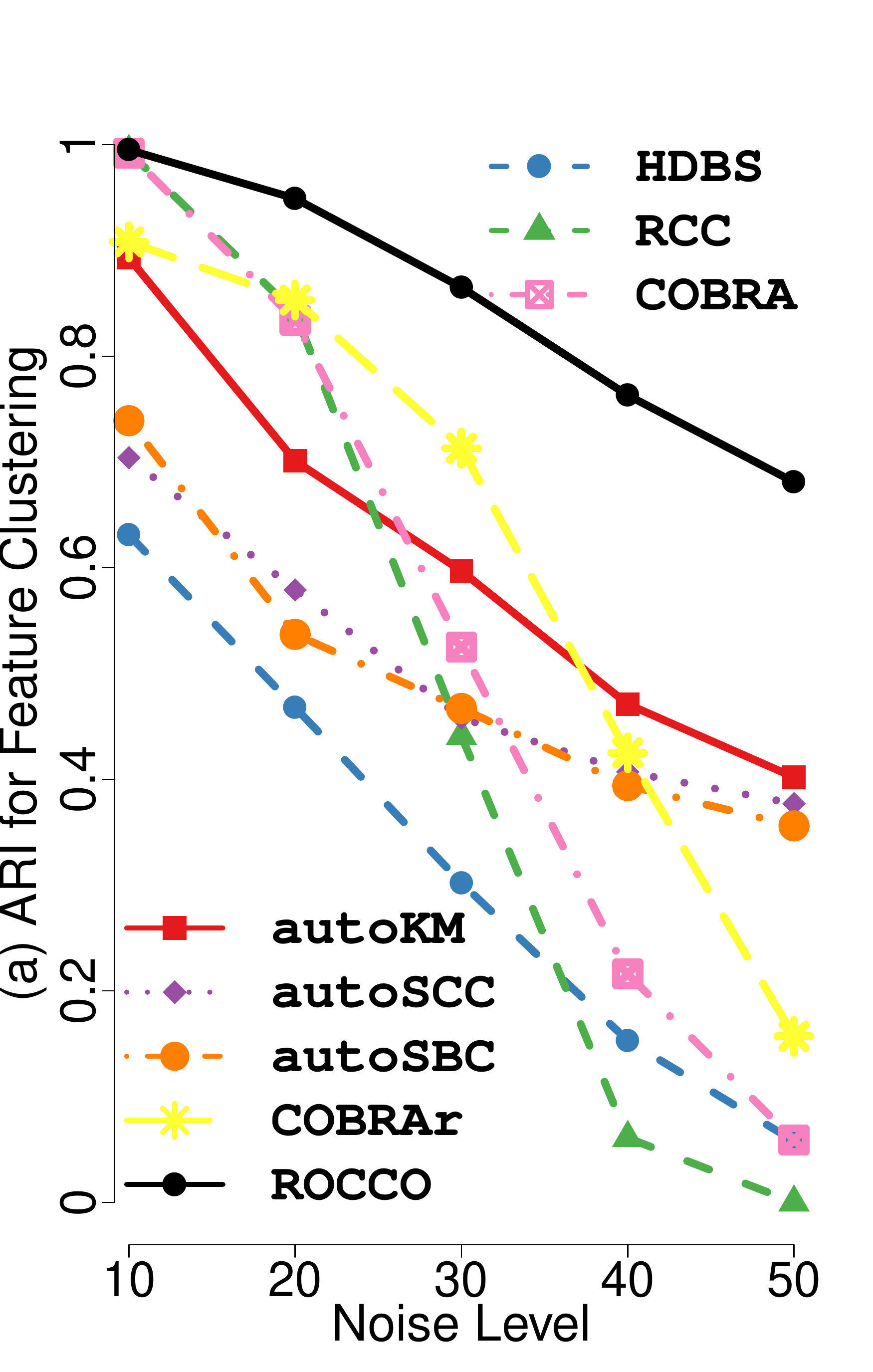}
		\label{} 
		\vspace{-0.7cm}
	\end{subfigure}
~~~~~~
	\begin{subfigure}[b]{0.38\linewidth}
		\centering
		\includegraphics[trim=0.0cm 0cm 1cm 2.5cm, width=\linewidth,clip=TRUE]{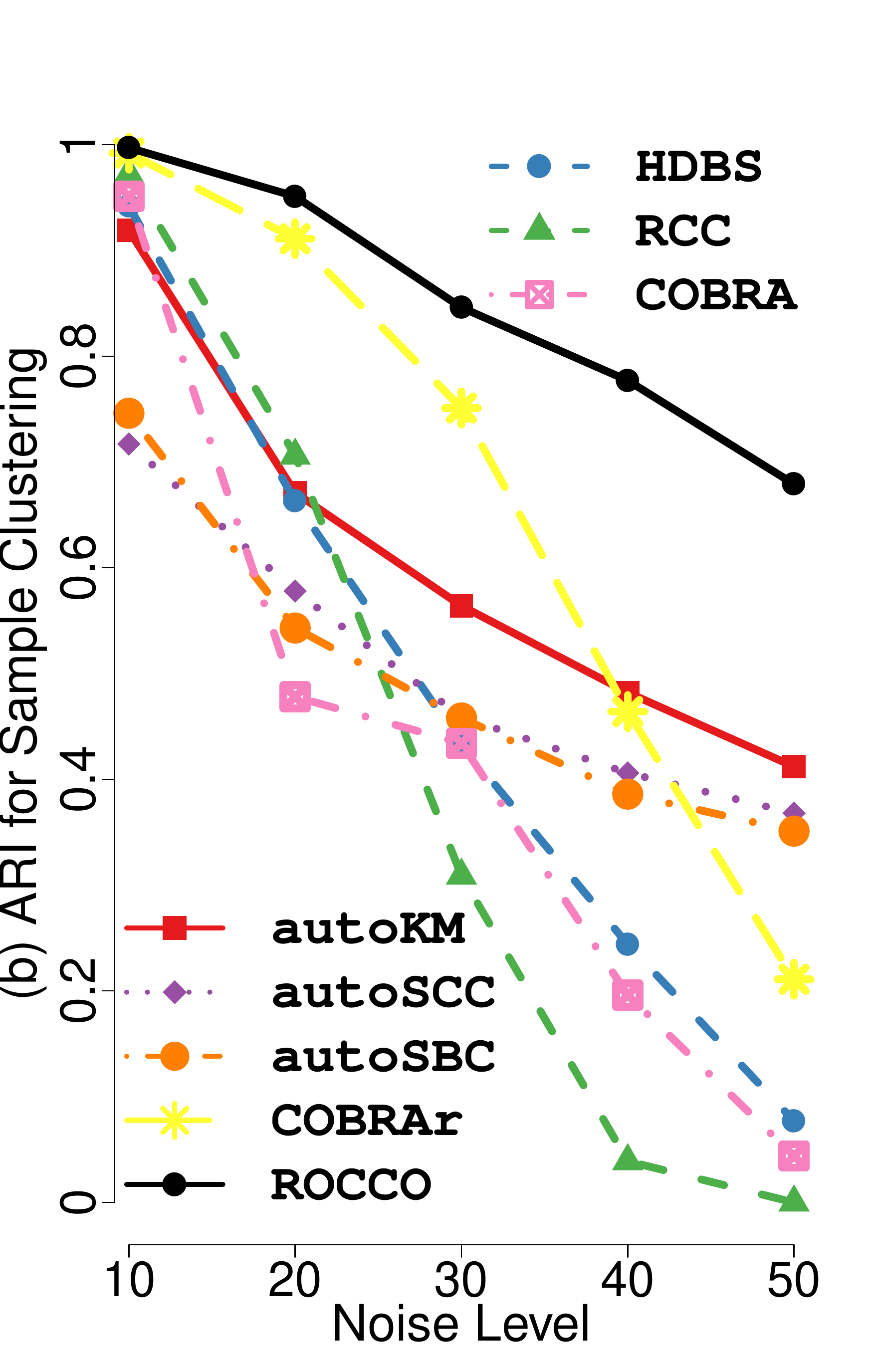}
		\label{} 
		\vspace{-0.7cm}
	\end{subfigure}
	\caption{Noise-based benchmark on Synthetic datasets over either samples and features using ARI. Best viewed in color.}
	\label{fig:Syn}
	\vspace{-2mm}
\end{figure*}
\subparagraph*{\textbf{Robustness (Q4).}}
\texttt{ROCCO} possesses two hyperparameters used to construct the initial graph: (i) $K$ for the mutual $K$-nearest neighbor graph; and (ii) the used distance function. We empirically evaluate its robustness over two grid of values defined as $K=\{7..13\}$ and \bera{distance\_Function = \{(E)uclidean, (M)anhattan, (C)osine\}}, respectively. Fig. \ref{fig:sensitivity} illustrates the results obtained on two real-world datasets: \bera{Classic3} and \bera{Carcinom}.  \newline
\begin{figure}[!t]
	\includegraphics[trim=1.25cm 0.65cm 1.0cm 1.5cm, width=0.55\linewidth,clip=TRUE]{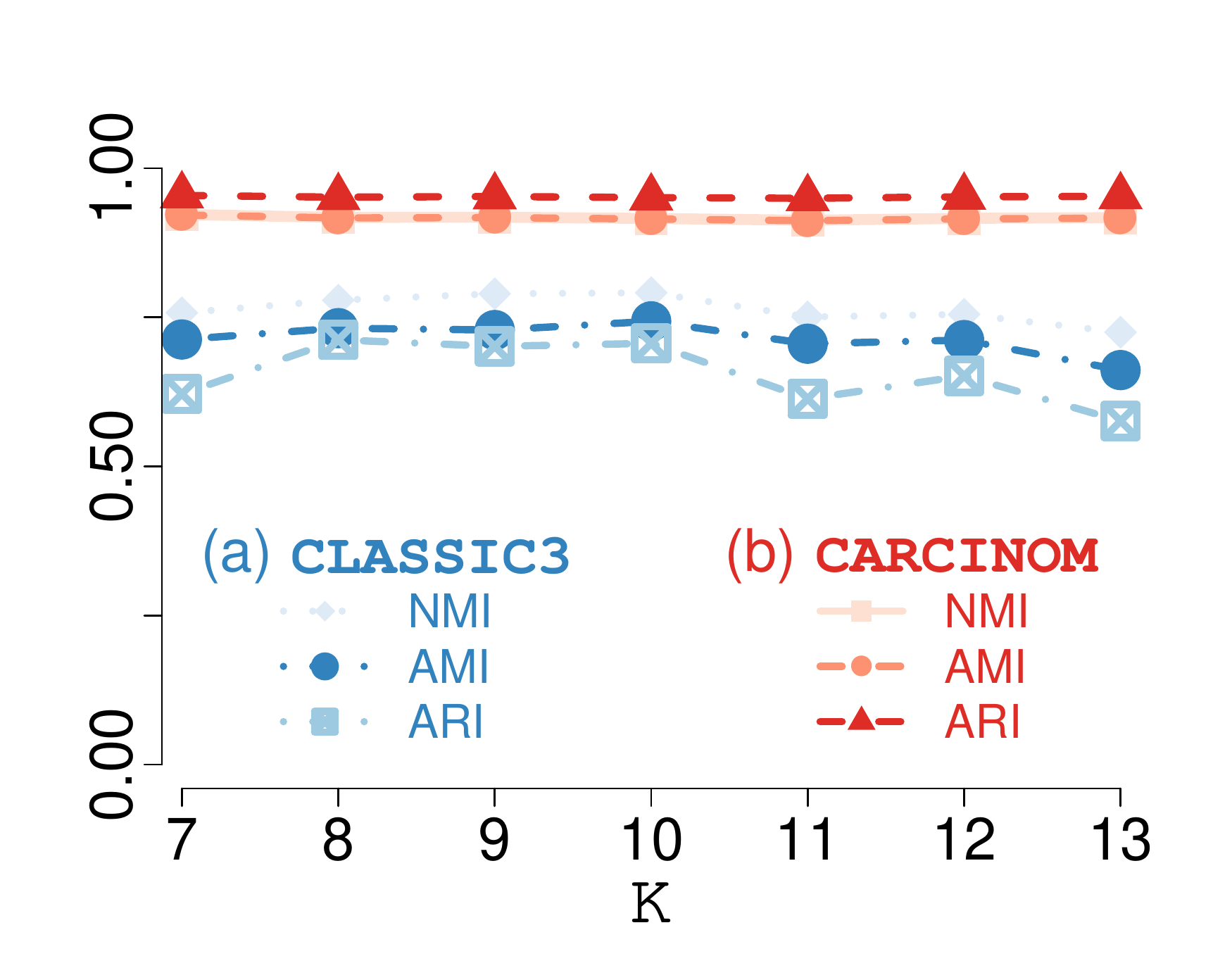}  
	\scalebox{1.3}{\begin{tabular}[b]{l||r|r|r}
			\cline{1-4}\noalign{\smallskip}
			& \textbf{NMI} & \textbf{AMI} & \textbf{ARI} \\
			\cline{1-4}\noalign{\smallskip}
			\cline{1-4}\noalign{\smallskip}
			\bera{E} (a) & 0.9157 & 0.9148 & 0.9541 \\
			\bera{M} (a) & 0.9144 & 0.9136 & 0.9525 \\
			\bera{C} (a) & 0.9157 & 0.9148 & 0.9507 \\ \hline
			VAR & 5.6$\mathrm{e}^{-7}$ & 4.8$\mathrm{e}^{-7}$ & 1.6$\mathrm{e}^{-7}$ \\ \hline\hline
			\bera{E} (b) & 0.7911 & 0.7432 & 0.7063 \\
			\bera{M} (b) & 0.7911 & 0.7432 & 0.7063 \\
			\bera{C} (b) & 0.7911 & 0.7432 & 0.7063 \\ \hline
			VAR & 0.0000 & 0.0000 & 0.0000 \\ \hline\hline
	\end{tabular}}
	\caption{Sensitivity Analysis to \texttt{ROCCO}'s hyperparameters used to build the initial graph, considering both (a) \bera{CLASSIC3} and (b) \bera(CARCINOM) datasets. The chart on left depicts an analysis of $K$, the number of neighbors. The Table on right depicts the results for the (E)uclidean, (M)anhattan and (C)osine distance metrics, as well as their respective VARiance.}
	\label{fig:sensitivity}
\end{figure}
\subparagraph*{\textbf{Scalability (Q5).}}
\begin{figure}[!t]
	\centering    
	\begin{subfigure}[b]{0.37\linewidth}
		\centering
		\vspace{-0.4cm}
		\includegraphics[trim=0.25cm 1.05cm 0.97cm 0.8cm, width=\linewidth,clip=TRUE]{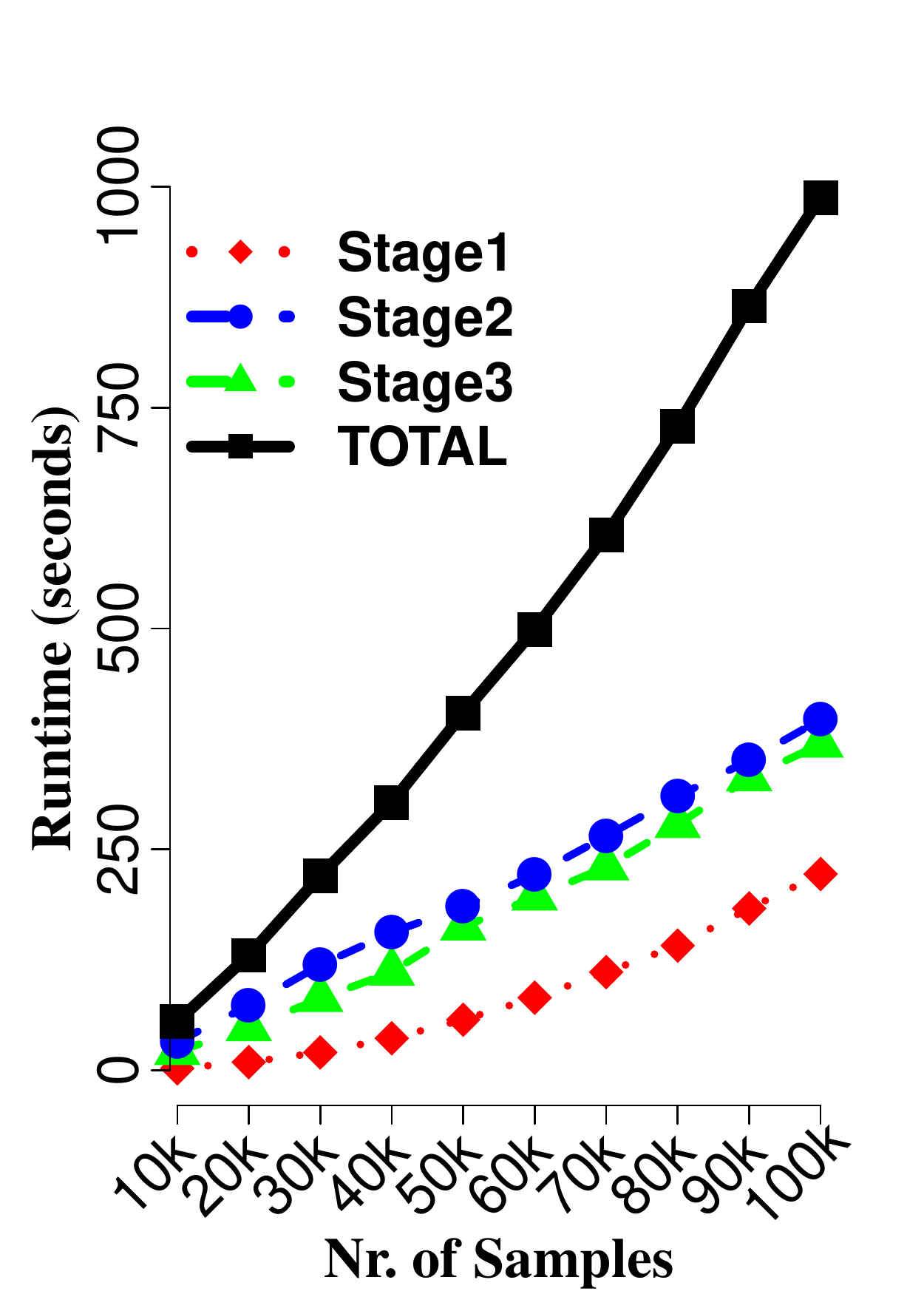}
		\label{}
		\vspace{-0.7cm}
	\end{subfigure}
	\begin{subfigure}[b]{0.75\linewidth}
		\centering
		\includegraphics[trim=0.25cm 0.9cm 0.0cm 0.0cm, width=\linewidth,clip=TRUE]{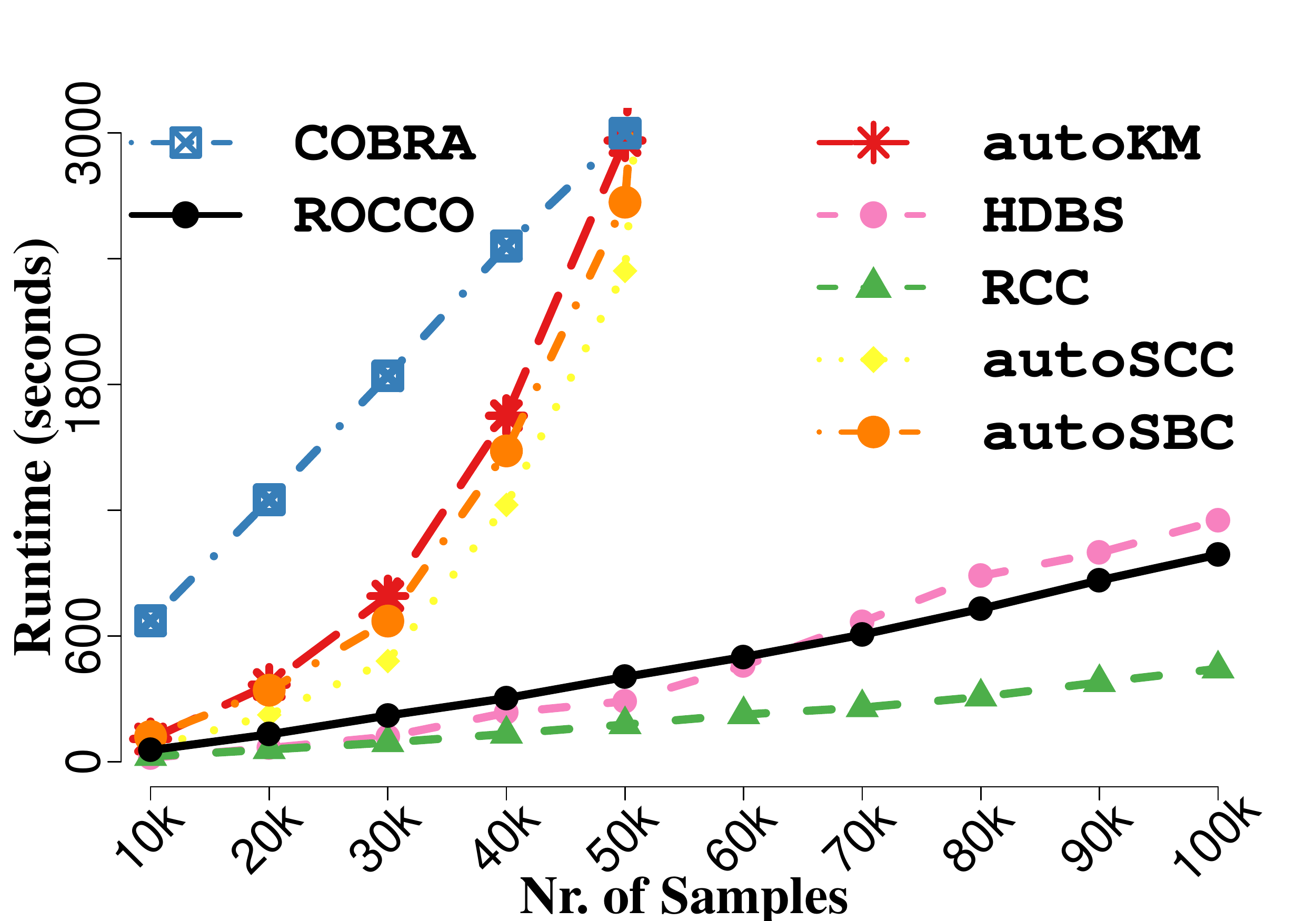}
		\label{} 
		\vspace{-0.5cm}
	\end{subfigure}
	\caption{Empirical Scalability Evaluation using synthetic data with different sample sizes (x-axis). Top chart depicts the analysis of runtime complexity of the 3 stages of \texttt{ROCCO}: (1) graph generation, (2) representation learning and (3) clustering assignment. Bottom one benchmarks \texttt{ROCCO} \textit{vs.} its comparison partners. Best viewed in color.}
	\label{fig:scalability}
	\vspace{-2mm}
\end{figure}
To evaluate scalability empirically, we generated 10 datasets with fixed feature size (i.e. $p=100$), $\{10..100\}\times10^3$ samples and a residual noise level (i.e. 10). These datasets are used to benchmark \texttt{ROCCO} against its comparison partners. Fig. \ref{fig:scalability} illustrates the obtained results, adding also detailed analysis about the running times of the different \texttt{ROCCO} stages. Note that the \texttt{auto*} methods include the Silhouette-based model selection procedure.

\subsection{Discussion} 
In this section, we discuss the results to see whether they answer the five research questions we proposed to address with these  experiments. \newline

\subparagraph*{\textbf{Representation Learning (Q1).}} From the heatmaps in Figs. \ref{fig:intro_heatmap} and \ref{fig:heatmap}, it can be clearly seen that the representation learned by \texttt{ROCCO} (Fig. \ref{fig:heatmap}-(f)) on the synthetic dataset successfully discovers the underlining Co-clustering structure (Fig. \ref{fig:intro_heatmap}-(a)). On the other hand, the representations learned by \texttt{COBRA} do not highlight that structure so well - even when varying $RH$ exhaustively (Fig. \ref{fig:heatmap}-(a,b,c)). In the same line, \texttt{RCC} also fails to discover such structure when addressing directly either samples or features individually (Fig. \ref{fig:heatmap}-(d,e)).

Table \ref{tab:results} highlights that \texttt{ROCCO} has top-performance on $6$ and $5$ out of the $9$ datasets, the second best on $2$ and $3$ other datasets, accordingly to NMI and AMI, respectively. We have a similar observation on the results of ARI as well\footnote{\label{note_space_limitation} These results are omitted due to space limitation.}. In sum, \texttt{ROCCO} ranks in the first place on both metrics, followed by \texttt{RCC} and \texttt{COBRA}. 

These results show the essential of the objective formulation of \texttt{ROCCO} and its advancement compared to SoA methods regarding representation learning for CoC. \newline

\subparagraph*{\textbf{Real-World Problems (Q2, Q3).}} Not surprisingly, after observing Fig. \ref{fig:Syn}, we can conclude that all methods performance is inversely proportional to the noise level (NMI and AMI results are similar\footnoteref{note_space_limitation}). Yet, \texttt{ROCCO} is by far the least affected method: it provides an ARI of $0.68$ on either samples and features for the highest noise level, where the second best method (i.e. \texttt{autoKM}) only achieves an ARI of $0.41$. Moreover, \texttt{COBRAr} also outperforms \texttt{COBRA} in most cases. This fact shows the effectiveness of \texttt{ROCCO}'s clustering assignment strategy.

Table \ref{tab:results2} uncovers that \texttt{ROCCO} also achieves high performance when facing real-world datasets from multiple domains. \texttt{ROCCO} outperforms its counterparts on $4$ and $5$ out of the $9$ datasets, on NMI and AMI, respectively. Moreover, it is possible to conclude that it achieves almost always the second best on all the remaining datasets. Naturally, \texttt{ROCCO} ranks in the first place on both metrics, followed by \texttt{COBRA} and \texttt{autoKM} for NMI and AMI separately. Please note that \texttt{COBRA} provides good NMI results, but poor AMI results, since it typically produces a large number of clusters. 

These results are illustrative of the effectiveness of \texttt{ROCCO} on real-world CoC problems. However, the clustering assignment of \texttt{ROCCO} is fully data dependent - which may raise some issues. Consider the following example: if we use \texttt{RCC} to generate the final clustering of the representations learned on sparse datasets (e.g. text documents), we will obtain single points clusters\footnote{ This observation is also true when applying RCC directly on sparse datasets.}. To tackle this challenge, we proposed (in Section \ref{sec:algorithm}) to apply \texttt{autoKM} instead after representation learning. Notoriously, this approach achieves promising results (Table \ref{tab:results2}). However, this comes at the cost of increasing both \texttt{ROCCO}'s runtime (Silhouette-based model selection) and its sensitivity to outliers (inherited from \texttt{KM}). \newline

\subparagraph*{\textbf{Robustness (Q4).}} Fig. \ref{fig:sensitivity} shows that \texttt{ROCCO} performs stable with near zero variance across different distance functions on both evaluated datasets. Regarding $K$, \texttt{ROCCO} performs stable on \bera{CARCINOM} dataset with all $K$'s and on \bera{CLASSIC3} with $K=8,9,10$. However, it degrades its performance slightly with other $K$ values. The sensitivity may also be inherited from the usage of \texttt{KM} in the clustering assignment stage (for sparse datasets). Overall, we can say that \texttt{ROCCO}'s performance is robust against its hyperparameters. \newline

\subparagraph*{\textbf{Scalability (Q5).}} Top chart in Fig. \ref{fig:scalability} depicts the runtime of 3 stages of \texttt{ROCCO}, where the runtime of stage $1$ (graph generation) scales quadratically, while stages $2$ (representation learning) and $3$ (clustering assignment) scale nearly linearly. Overall, \texttt{ROCCO}'s total runtime complexity is slightly over a linear one, since the quadratical part is still the cheapest even for $100k$ samples. 

Bottom chart in Fig. \ref{fig:scalability} depicts that \texttt{ROCCO}, \texttt{RCC}, \texttt{COBRA}, and \texttt{HDBS} scale nearly linear, while the runtime of \texttt{auto*} methods scales clearly in a quadratic scale. As mentioned before, the quadratic complexity of all \texttt{auto*} methods is due to Silhouette-based model selection. \texttt{KM}, \texttt{SCC}, and \texttt{SBC} scale linearly when the number of clusters is given beforehand\footnoteref{note_space_limitation}. Although possessing a linear runtime curve, \texttt{COBRA} turns to be quite expensive when compared to remaining partners. This slope derives from its heuristic to tune $RH$, which involves to re-run the algorithm multiple times.

By the above-mentioned reasons, all the five research questions can be answered affirmatively.

		\section{Related Work}
In this Section, we briefly review the SoA with respect to CoC, as well as to two related topics: Parameter-free and Continuous Clustering.
\subsection{Co-Clustering}
CoC has been extensively studied in the literature, which works on high-dimensional data and increases the interpretability compared to classical clustering. \texttt{SCC} \cite{dhillon2001co} is one of the earliest works on co-clustering documents and words. It treats the input data matrix as a bipartite graph and performs the normalized cut on the graph. \texttt{SBC} from Kluger \textit{et al.} \cite{kluger2003spectral} is another SoA method for CoC gene expression data. \texttt{SBC} performs SVD to approximate the original matrix followed by \texttt{KM} to find the row and column clusters. Later, many CoC methods were proposed. Dhillon \textit{et al.} proposed an information theory based co-clustering method \cite{dhillon2003information}. Ding \textit{et al.} proposed a CoC method based on Nonnegative matrix tri-factorization \cite{ding2006orthogonal}. Shan and Banerjee proposed Bayesian co-clustering \cite{shan2008bayesian}. Du and Shen proposed a CoC method with graph regularized Nonnegative matrix tri-factorization \cite{du2013towards}. Nie \textit{et al.} proposed to learn a structured optimal bipartite graph for CoC \cite{nie2017learning}. However, all these methods require to set hyperparameters beforehand, e.g. the number of clusters, subspace dimensionality or regularizers. Commonly, these hyperparameters settings strongly affect the clustering results. Typically, there are no universal guidelines on how to select them in practice. In contrast to the abovementioned methods, \texttt{ROCCO} learns critical hyperparameters automatically, e.g. regularizers. Moreover, it is robust to other hyperparameters, e.g. $K$ for constructing $K$-nearest neighbor graphs.

\subsection{Parameter-free Clustering}
Clustering is an unsupervised learning task. Therefore, hyperparameter settings are crucial in practice. Many metrics are proposed to determine the number of clusters for \texttt{KM}. The Silhouette Coefficient has been proposed long ago \cite{rousseeuw1987silhouettes}. Yet, it is still one of the most widely used metrics. Wiwie \textit{et al.} \cite{wiwie2015comparing} studies the performance of different clustering metrics and show that the Silhouette Coefficient is the most suitable for bio-medical data. \texttt{HDBS} \cite{campello2013density} is a nearly-automatic density-based approach for clustering that eliminates \texttt{DBSCAN} \cite{ester1996density} need to cut-off the resulting dendrogram. Instead, the final clusters are determined by traversing a tree of condensed view splits, allowing clusters of different densities to be  formed.

The literature on parameter-free methods CoC is scarce. To the best of our knowledge, \texttt{HiCC} \cite{ienco2009parameter} is the only method that tries to address this issue in CoC. However, \texttt{HiCC} tries to learn the hierarchy of samples and features together in a parameter-free way - which is a problem distinct from the one in our scope.

\subsection{Continuous Clustering}
Continuous clustering algorithms attracted attention recently. The main idea is to convert the clustering problem to a continuous optimization problem, where a regularized graph-based representation is learned with good clustering structure. Firstly, \cite{lindsten2011just,hocking2011clusterpath} proposed to relax the \texttt{KM} clustering and hierarchical clustering to convex optimization problems. The newly learned representation owns a good cluster structure, so that clustering assignment can be achieved from the connected component of the calculated graph. Later, Chi and Lange \cite{chi2015splitting} proposed a splitting method to accelerate the solving of the convex optimization problem. All the above-mentioned methods use the convex ($L_2$ norm) function for regularization. Recently, Shah and Koltun \cite{Shah12092017} proposed \texttt{RCC}, which uses a non-convex (Geman-McClure) function for regularization. \texttt{RCC} achieves SoA clustering accuracy without setting regularization parameters. The proposed \texttt{ROCCO} can be seen as an extension of \texttt{RCC} to the CoC problem. \texttt{ROCCO} inherited the merits of \texttt{RCC} with high-quality clustering and no regularization parameter setting. Recently, Chi \textit{et al.} proposed a convex bi-clustering algorithm \texttt{COBRA} \cite{chic2017}, which is the most relevant method to the proposed \texttt{ROCCO}. \texttt{COBRA} also learns the representation regularized on both sample and feature graphs. However, \texttt{COBRA} adopts a convex ($L_2$ norm) regularization function, while the proposed \texttt{ROCCO} uses a non-convex function for the same purpose. Therefore, the two optimization problems are radically different. \texttt{ROCCO} proposed a new optimization method to solve the resulting non-convex (harder) problem. Finally, \texttt{ROCCO} convincingly outperforms \texttt{COBRA} in our experiments. 

\section{Final Remarks}
In this paper, we formulate CoC as a continuous non-convex optimization problem that settles on a graph-based representation of both samples and features. Further, we present an approach (i.e. \texttt{ROCCO}) that solves this problem with theoretical near-linear time complexity on the input data matrix. Extensive experiments show the effectiveness of the learned representations and the resulting clustering quality. Along with these characteristics, \texttt{ROCCO}'s robustness to hyperparameters, and its scalability on both synthetic and real-world datasets either matches and/or (mostly) outperforms SoA counterparts. \texttt{ROCCO} will serve as a parameter-free CoC method for large-scale problems, regardless of the characteristics of the input data or the application domain. 

\bibliography{acmart}

\newcommand{\etalchar}[1]{$^{#1}$}
\begin{thebibliography}{KMMCC16}

\bibitem[BCQY97]{brito1997connectivity}
M.~Brito, E.~Chavez, A.~Quiroz, and J.~Yukich.
\newblock Connectivity of the mutual k-nearest-neighbor graph in clustering and
  outlier detection.
\newblock {\em Statistics \& Probability Letters}, 35(1):33--42, 1997.

\bibitem[BPP08]{busygin2008}
S.~Busygin, O.~Prokopyev, and P.~Pardalos.
\newblock Biclustering in data mining.
\newblock {\em Computers \& Operations Research}, 35(9):2964--2987, 2008.

\bibitem[BS72]{bartels1972solution}
R.~Bartels and G.~Stewart.
\newblock Solution of the matrix equation ax+xb=c [f4].
\newblock {\em Communications of the ACM}, 15(9):820--826, 1972.

\bibitem[CAB17]{chic2017}
E.~Chi, G.~Allen, and R.~Baraniuk.
\newblock Convex biclustering.
\newblock {\em Biometrics}, 73(1):10--19, 2017.

\bibitem[CFS09]{chen2009fast}
J.~Chen, H.~Fang, and Y.~Saad.
\newblock Fast approximate knn graph construction for high dimensional data via
  recursive lanczos bisection.
\newblock {\em Journal of Machine Learning Research}, 10(Sep):1989--2012, 2009.

\bibitem[CKM{\etalchar{+}}14]{cohen2014solving}
M.~Cohen, R.~Kyng, G.~Miller, J.~Pachocki, R.~Peng, A.~Rao, and S.~Xu.
\newblock Solving sdd linear systems in nearly m log 1/2 n time.
\newblock In {\em Proceedings of the 46th Annual ACM symposium on Theory of
  computing}, pages 343--352. ACM, 2014.

\bibitem[CL15]{chi2015splitting}
E.~Chi and K.~Lange.
\newblock Splitting methods for convex clustering.
\newblock {\em Journal of Computational and Graphical Statistics},
  24(4):994--1013, 2015.

\bibitem[CMS13]{campello2013density}
R.~Campello, D.~Moulavi, and J.~Sander.
\newblock Density-based clustering based on hierarchical density estimates.
\newblock In {\em Pacific-Asia conference on Knowledge Discovery and Data
  mining (PAKDD)}, pages 160--172. Springer, 2013.

\bibitem[Dhi01]{dhillon2001co}
I.~Dhillon.
\newblock Co-clustering documents and words using bipartite spectral graph
  partitioning.
\newblock In {\em Proceedings of the 7th ACM SIGKDD International Conference on
  Knowledge Discovery and Data mining (KDD)}, pages 269--274. ACM, 2001.

\bibitem[DLPP06]{ding2006orthogonal}
C.~Ding, T.~Li, W.~Peng, and H.~Park.
\newblock Orthogonal nonnegative matrix t-factorizations for clustering.
\newblock In {\em Proceedings of the 12th ACM SIGKDD International Conference
  on Knowledge Discovery and Data mining (KDD)}, pages 126--135. ACM, 2006.

\bibitem[DML11]{dong2011efficient}
W.~Dong, C.~Moses, and K.~Li.
\newblock Efficient k-nearest neighbor graph construction for generic
  similarity measures.
\newblock In {\em Proceedings of the 20th International Conference on World
  Wide Web (WWW)}, pages 577--586. ACM, 2011.

\bibitem[DMM03]{dhillon2003information}
I.~Dhillon, S.~Mallela, and D.~Modha.
\newblock Information-theoretic co-clustering.
\newblock In {\em Proceedings of the 9th ACM SIGKDD International Conference on
  Knowledge Discovery and Data mining (KDD)}, pages 89--98. ACM, 2003.

\bibitem[DS13]{du2013towards}
L.~Du and Y.~Shen.
\newblock Towards robust co-clustering.
\newblock In {\em Proceedings of International Joint Conference on Artificial
  Intelligence (IJCAI)}, pages 1317--1322, 2013.

\bibitem[EKS{\etalchar{+}}96]{ester1996density}
M.~Ester, H.~Kriegel, J.~Sander, X.~Xu, et~al.
\newblock A density-based algorithm for discovering clusters in large spatial
  databases with noise.
\newblock In {\em Proceedings of the 2nd ACM SIGKDD International Conference on
  Knowledge Discovery and Data mining (KDD)}, volume~96, pages 226--231, 1996.

\bibitem[HA85]{hubert1985comparing}
L.~Hubert and Phipps Arabie.
\newblock Comparing partitions.
\newblock {\em Journal of classification}, 2(1):193--218, 1985.

\bibitem[HJBV11]{hocking2011clusterpath}
T.~Hocking, A.~Joulin, F.~Bach, and J.~Vert.
\newblock Clusterpath an algorithm for clustering using convex fusion
  penalties.
\newblock In {\em 28th International Conference on Machine Learning (ICML)},
  2011.

\bibitem[HP99]{hofmann1999}
T.~Hofmann and J.~Puzicha.
\newblock Latent class models for collaborative filtering.
\newblock In {\em Proceedings of International Joint Conference on Artificial
  Intelligence (IJCAI)}, volume~99, 1999.

\bibitem[HS52]{hestenes1952methods}
M.~Hestenes and E.~Stiefel.
\newblock {\em Methods of conjugate gradients for solving linear systems},
  volume~49.
\newblock NBS, 1952.

\bibitem[IMT11]{KoutisMT11}
K.~Ioannis, G.~Miller, and D.~Tolliver.
\newblock Combinatorial preconditioners and multilevel solvers for problems in
  computer vision and image processing.
\newblock {\em Computer Vision and Image Understanding}, 115(12):1638 -- 1646,
  2011.

\bibitem[IPM09]{ienco2009parameter}
Dino Ienco, Ruggero~G Pensa, and Rosa Meo.
\newblock Parameter-free hierarchical co-clustering by n-ary splits.
\newblock In {\em Joint European Conference on Machine Learning and Knowledge
  Discovery in Databases}, pages 580--595. Springer, 2009.

\bibitem[KBCG03]{kluger2003spectral}
Y.~Kluger, R.~Basri, J.~Chang, and M.~Gerstein.
\newblock Spectral biclustering of microarray data: co-clustering genes and
  conditions.
\newblock {\em Genome research}, 13(4):703--716, 2003.

\bibitem[KMMCC16]{khiari2016}
J.~Khiari, L.~Moreira-Matias, V.~Cerqueira, and O.~Cats.
\newblock Automated setting of bus schedule coverage using unsupervised machine
  learning.
\newblock In {\em Pacific-Asia Conference on Knowledge Discovery and Data
  Mining (PAKDD)}, pages 552--564. Springer, 2016.

\bibitem[KMP11]{koutis2011nearly}
I.~Koutis, G.~Miller, and R.~Peng.
\newblock A nearly-m log n time solver for sdd linear systems.
\newblock In {\em Foundations of Computer Science (FOCS), 2011 IEEE 52nd Annual
  Symposium on}, pages 590--598. IEEE, 2011.

\bibitem[KOSZ13]{kelner2013simple}
J.~Kelner, L.~Orecchia, A.~Sidford, and Z.~Zhu.
\newblock A simple, combinatorial algorithm for solving sdd systems in
  nearly-linear time.
\newblock In {\em Proceedings of the 45th annual ACM symposium on Theory of
  computing}, pages 911--920. ACM, 2013.

\bibitem[Lan95]{lang1995newsweeder}
K.~Lang.
\newblock Newsweeder: Learning to filter netnews.
\newblock In {\em Machine Learning Proceedings 1995}, pages 331--339. Elsevier,
  1995.

\bibitem[LOL11]{lindsten2011just}
F.~Lindsten, H.~Ohlsson, and L.~Ljung.
\newblock {\em Just relax and come clustering!: A convexification of k-means
  clustering}.
\newblock Link{\"o}ping University Electronic Press, 2011.

\bibitem[LW03]{liu2003}
J.~Liu and W.~Wang.
\newblock Op-cluster: Clustering by tendency in high dimensional space.
\newblock In {\em Data Mining, 2003. ICDM 2003. Third IEEE International
  Conference on}, pages 187--194. IEEE, 2003.

\bibitem[MFI15]{mobahi2015theoretical}
H.~Mobahi and J.~Fisher~III.
\newblock A theoretical analysis of optimization by gaussian continuation.
\newblock In {\em AAAI}, pages 1205--1211, 2015.

\bibitem[MH17]{mcinnes2017accelerated}
L.~McInnes and J.~Healy.
\newblock Accelerated hierarchical density based clustering.
\newblock In {\em Data Mining Workshops (ICDMW), 2017 IEEE International
  Conference on}, pages 33--42. IEEE, 2017.

\bibitem[MP16]{maurus2016}
S.~Maurus and C.~Plant.
\newblock Skinny-dip: Clustering in a sea of noise.
\newblock In {\em Proceedings of the 22nd ACM SIGKDD International Conference
  on Knowledge Discovery and Data mining (KDD)}, pages 1055--1064. ACM, 2016.

\bibitem[NHCD10]{nie2010efficient}
F.~Nie, H.~Huang, X.~Cai, and C.~Ding.
\newblock Efficient and robust feature selection via joint l2, 1-norms
  minimization.
\newblock In {\em Advances in Neural Information Processing Systems (NIPS)},
  pages 1813--1821, 2010.

\bibitem[NWDH17]{nie2017learning}
F.~Nie, X.~Wang, C.~Deng, and H.~Huang.
\newblock Learning a structured optimal bipartite graph for co-clustering.
\newblock In {\em Advances in Neural Information Processing Systems (NIPS)},
  pages 4132--4141, 2017.

\bibitem[PVG{\etalchar{+}}11]{scikit-learn}
F.~Pedregosa, G.~Varoquaux, A.~Gramfort, et~al.
\newblock Scikit-learn: Machine learning in {P}ython.
\newblock {\em Journal of Machine Learning Research}, 12:2825--2830, 2011.

\bibitem[Rou87]{rousseeuw1987silhouettes}
P.~Rousseeuw.
\newblock Silhouettes: a graphical aid to the interpretation and validation of
  cluster analysis.
\newblock {\em Journal of computational and Applied Mathematics}, 20:53--65,
  1987.

\bibitem[SB08]{shan2008bayesian}
H.~Shan and A.~Banerjee.
\newblock Bayesian co-clustering.
\newblock In {\em Data Mining, 8th IEEE International Conference on (ICDM)},
  pages 530--539. IEEE, 2008.

\bibitem[SK17]{Shah12092017}
S.~Shah and V.~Koltun.
\newblock Robust continuous clustering.
\newblock {\em Proceedings of the National Academy of Sciences (PNAS)},
  114(37):9814--9819, 2017.

\bibitem[TSS02]{tanay2002}
A.~Tanay, R.~Sharan, and R.~Shamir.
\newblock Discovering statistically significant biclusters in gene expression
  data.
\newblock {\em Bioinformatics}, 18(1):S136--S144, 2002.

\bibitem[VEB10]{vinh2010information}
N.~Vinh, J.~Epps, and J.~Bailey.
\newblock Information theoretic measures for clusterings comparison: Variants,
  properties, normalization and correction for chance.
\newblock {\em Journal of Machine Learning Research}, 11(Oct):2837--2854, 2010.

\bibitem[WBR15]{wiwie2015comparing}
Christian Wiwie, Jan Baumbach, and Richard R{\"o}ttger.
\newblock Comparing the performance of biomedical clustering methods.
\newblock {\em Nature methods}, 12(11):1033, 2015.

\end{thebibliography}
\bibliographystyle{alpha}

\end{document}